\documentclass[11pt]{article}

\usepackage[a4paper,margin=1in]{geometry}
\usepackage[numbers,sort&compress]{natbib}
\usepackage{microtype}
\usepackage{amsmath,amssymb,amsfonts,amsthm,mathrsfs,mathtools}
\usepackage{algorithm}
\usepackage{algpseudocode}
\usepackage{booktabs}
\usepackage{multirow}
\usepackage{array}
\usepackage{longtable}
\usepackage{graphicx}
\usepackage{subcaption}
\usepackage[svgnames]{xcolor}
\usepackage{enumitem}
\usepackage[section]{placeins}
\usepackage{needspace}
\usepackage{siunitx}
\usepackage{float}
\usepackage[hypertexnames=false]{hyperref}

\usepackage[capitalize,nameinlink]{cleveref}

\graphicspath{{figures/}}
\allowdisplaybreaks[4]
\numberwithin{equation}{section}
\newtheorem{theorem}{Theorem}[section]

\newtheorem{proposition}[theorem]{Proposition}
\newtheorem{lemma}[theorem]{Lemma}
\newtheorem{corollary}[theorem]{Corollary}
\theoremstyle{definition}

\theoremstyle{remark}
\newtheorem{remark}[theorem]{Remark}

\newcolumntype{L}[1]{>{\raggedright\arraybackslash}p{#1}}
\newcommand{\R}{{\mathbb R}}

\renewcommand{\leq}{\leqslant}
\renewcommand{\geq}{\geqslant}
\def\epsilon{\varepsilon}
\def\phi{\varphi}
\def\supp{\operatorname{supp}}

\Crefname{figure}{Figure}{Figures}
\crefformat{equation}{\textup{#2(#1)#3}}
\crefdefaultlabelformat{#2\textup{#1}#3}
\setlist[enumerate]{leftmargin=.35in}
\setlist[itemize]{leftmargin=.35in}
\title{Softmax Routing Dynamics in Mixture-of-Experts Models:\\Collapse and Recovery}
\author{
Shengze Xu$^{1}$ \quad
Maolin Deng$^{1}$ \quad
Tieyong Zeng$^{2,3}$\thanks{Corresponding author: \texttt{tieyongzeng@bnbu.edu.cn}}\\[0.5em]
$^{1}$The Chinese University of Hong Kong\\
$^{2}$Beijing Normal-Hong Kong Baptist University \quad
$^{3}$Guangzhou Nanfang College
}
\date{}
\hypersetup{
  colorlinks=true,
  citecolor=DarkGreen,
  linkcolor=DarkRed,
  urlcolor=DarkBlue,
  pdftitle={Softmax Routing Dynamics in Mixture-of-Experts Models: Collapse and Recovery},
  pdfauthor={Shengze Xu; Maolin Deng; Tieyong Zeng},
  pdfsubject={Geometry, time scales, and recovery mechanisms of softmax routing in mixture-of-experts models}
}

\begin{document}
\maketitle

\begin{abstract}
Softmax routing in mixture-of-experts models can concentrate probability on a subset of experts, suppressing access to the others. We study how this concentration forms, why it persists, and how suppressed probabilities recover. For fixed experts, a stationary-loss barrier identifies the face supporting the optimal routed output. In a common-target square-loss model with freely trainable expert outputs, slow adaptation tracks fixed-expert dynamics on finite horizons, while every fixed positive adaptation rate yields an interior routing limit. An exact normal-velocity decomposition separates complementary-expert covariance from relative gradient advantage. Under positive two-sided power bounds along recovery paths, uniform over initial masses, the resulting coefficient's order of vanishing gives logarithmic or power-law recovery times. Integrations of constructed fields reproduce these orders. An exact discrete log-odds identity relates neural updates to threshold crossing. In paired Transformer-MoE training, removing a class-conditional router-gradient contrast delays confirmed recovery and reduces accumulated log-odds on fixed tokens, both with and without capacity constraints.

\end{abstract}

\noindent\textbf{Keywords:} mixture of experts, softmax routing, routing collapse, gradient flow, routing recovery

\section{Introduction}

Sparse mixture-of-experts (MoE) models increase capacity by activating only a subset of expert subnetworks for each input, enabling conditional computation at large scale
\cite{shazeer2017outrageously,lepikhin2020gshard,fedus2022switch,du2022glam}. A learned router assigns probabilities to experts, while top-$k$ selection and capacity constraints determine the executed assignments. During training, self-reinforcing routing preferences can concentrate assignments on a few experts, a form of routing collapse
\cite{shazeer2017outrageously,cheng2026emoe}.

We study \emph{routing concentration}: small pre-dispatch probability mass outside a proper coordinate face. Asymptotic concentration means that this mass tends to zero. Severe concentration suppresses access to some experts for a given input and constitutes the probability-space aspect of \emph{routing collapse}. Finite logits keep probabilities positive, so finite-time collapse appears as near-boundary concentration. We call the return of a suppressed probability to a prescribed positive threshold \emph{recovery}.

Load-balancing losses encourage balanced aggregate expert use
\cite{shazeer2017outrageously,fedus2022switch,zoph2022st}. Yet individual inputs can concentrate on different experts while maintaining balanced aggregate loads. Routing probabilities also evolve jointly with expert parameters and input representations. Persistent concentration and subsequent reactivation observed during MoE training
\cite{sha2026sparse}
therefore raise three questions: which simplex face is selected as routing concentrates? Why can concentration persist when experts are trainable? What determines the time required for a suppressed expert to recover?

Near a simplex boundary, a small probability can respond slowly even to a favorable gradient signal. Continuous-time gradient flow on softmax logits makes this effect explicit. For a routing vector $p$, let $J(p)=\operatorname{diag}(p)-pp^\top$ be the softmax Jacobian and $g$ the probability gradient of the half-loss. Euclidean logit-gradient flow and its probability-space form are
\[
\dot q=-J(p)g,\qquad \dot p=-J(p)^2g.
\]
At $p_j=0$, the $j$th row of $J(p)$ vanishes. Under bounded gradient signals, normal probability motion therefore vanishes as $p_j\downarrow0$; we call this \emph{boundary degeneracy}. It is a property of the vector field, whereas concentration describes the routing state. The direction and order of normal motion depend on the complementary routing state and gradient field. Related Jacobian structures appear in value--softmax dynamics
\cite{varre2026gradient}; here we characterize how this geometry governs routing concentration, persistence, and recovery.

\paragraph{Main results}
For fixed experts, the optimal routed output is the target's projection onto the expert convex hull. Below the smallest suboptimal stationary loss, every limit point represents this projection and is supported on experts in the minimal hull face containing it. Geometric separation gives bounds on outside-face mass and neighborhood entry times (\Cref{sec:theory}).

Expert adaptation introduces a second time scale. As the expert-update rate tends to zero, the joint flow approaches the fixed-expert trajectory on each fixed finite horizon. For every fixed positive rate in the common-target square-loss model with freely trainable outputs, however, residual decay makes the logits converge to finite limits and routing to an interior point. Thus concentration can persist over finite horizons even though the long-time limit is interior (\Cref{sec:coupled-main}).

Our main recovery result separates complementary-expert covariance from relative gradient advantage in an exact normal-velocity decomposition (\Cref{sec:boundary-recovery}). For a suppressed probability $x=p_j$, it gives $\dot x=x(1-x)^2\Gamma_j$. The order of vanishing of $\Gamma_j$ governs the recovery time $T_\theta(\epsilon)$ from initial mass $\epsilon$ to a fixed threshold $\theta$. For fixed $\alpha\geq0$, positive two-sided comparisons along the entire recovery path, with constants independent of $\epsilon$, yield
\begin{equation*}
\Gamma_j\asymp x^\alpha
\quad\Longrightarrow\quad
T_\theta(\epsilon)\asymp
\begin{cases}
\log(1/\epsilon),&\alpha=0,\\
\epsilon^{-\alpha},&\alpha>0,
\end{cases}
\qquad \epsilon\downarrow0.
\end{equation*}
For two experts under the same advantage signal bounded above and away from zero, squared-Jacobian, single-Jacobian, and Euclidean projected flows yield inverse-probability, logarithmic, and bounded recovery times, respectively.

For input-conditioned routers, an affine separation criterion describes initial routing differences, and an exact discrete log-odds identity relates threshold crossing to individual neural updates (\Cref{sec:neural-bridge}). Paired Transformer-MoE interventions show that removing a class-conditional router-gradient contrast delays confirmed recovery and reduces accumulated log-odds on fixed tokens, both with and without capacity constraints (\Cref{sec:experiments}).

Code is available at \url{https://github.com/shengze-xu/softmax-routing-dynamics}.

\Needspace{5\baselineskip}

\subsection{Related work}
\label{sec:related-work}

\paragraph{Routing, allocation, and statistical recovery}
Learned gating originated in adaptive and hierarchical mixtures
\cite{jacobs1991adaptive,jordan1994hierarchical}. Modern methods study expert assignment, balanced utilization, expert choice, and differentiable dispatch
\cite{lewis2021base,zhou2022mixture,hazimeh2021dselect,puigcerver2024sparse,wang2025remoe}.
Expert-wise biases control aggregate loads
\cite{wang2024auxiliary}, and assignment dynamics admit primal--dual formulations
\cite{han2025theoretical}. Load balance, input specialization, and token-level route quality are distinct allocation properties
\cite{guo2026advancing,wang2026myth,yoon2026experts}; assignments can become established early in pre-training
\cite{xue2024openmoe}. Statistical analyses address identifiability and estimation
\cite{chen1995optimal,allman2009identifiability,nguyen2024temperature,nguyen2024least}
or recovery of latent clusters and tasks
\cite{chen2022towards,dikkala2023benefits,kawata2025mixture,liao2025guided}.
Stochastic-optimization analyses quantify how routing sparsity affects overall convergence
\cite{zhurethinking}. We study the time a suppressed routing probability takes to regain a prescribed threshold during optimization.

\paragraph{Concentration observables and softmax dynamics}
Representation clustering around expert centroids
\cite{chi2022representation}, concentration of expert-selection paths across layers
\cite{gu2026path}, and the effective rank of routing-similarity graphs across reasoning rollouts
\cite{chen2026concentration} describe observables distinct from our outside-face mass for an individual probability vector.
Optimization geometry determines how gradients move that vector: replicator dynamics and entropic mirror descent induce a single softmax Jacobian in probability space
\cite{hofbauer1998evolutionary,sandholm2010population,beck2003mirror}, whereas Euclidean logit-gradient flow induces its square.
Tabular softmax policy-optimization rates depend on initialization and bounds on optimal-action probabilities
\cite{mei2020global}. Face selection also relates to active-set identification
\cite{wright1993identifiable}. Our stationary-loss barrier combines convex projection and loss dissipation to identify the hull face supporting the limiting output; exposing directions then control outside-face mass.

\paragraph{Trainable value--softmax dynamics}
Varre, Rofin, and Flammarion \cite{varre2026gradient} study joint value--softmax flow. Their square-loss analysis gives exponential residual decay and finite cumulative polarization; the latter is derived under a zero-initialized value matrix. We use the same dissipation mechanism to establish convergence of logits and expert outputs for arbitrary finite initialization and every positive expert-update rate, with explicit tail and probability-floor bounds (\Cref{thm:jointgradientinteriorfreeze}). Combining this result with the slow-expert approximation separates finite-horizon concentration from the interior long-time limit. Recovery is then governed by the normal coefficient in \Cref{sec:boundary-recovery}.

\paragraph{Router--expert coupling}
Joint training couples routing preferences to expert adaptation. Kiselev
\cite{kiselev2026minimal} derives symmetry-breaking bifurcations in a two-expert reinforcement-and-decay model.
Ahrac, Hochwald, and Geva
\cite{ahrac2026routers} analyze shared token directions in router and expert gradients and the effects of load balancing on their alignment.
We study how a class-conditional target--comparator gradient contrast affects subsequent recovery. Matched neural branches estimate the response to removing this contrast, and the discrete log-odds identity relates the observed routing changes to individual threshold crossing.

\Needspace{4\baselineskip}
\section{Problem setup and preliminaries}
\label{sec:model}

\subsection{Softmax routing as gradient flow}

For \(n\) experts, let \(p\in\Delta^{n-1}\) denote the routing vector, where
\(\Delta^{n-1}=\{p\in\R_+^n:\mathbf 1^\top p=1\}\) and \(\mathbf1\) is the
all-ones vector; \(\R_+=[0,\infty)\) and \(\top\) denotes transpose.
The relative interior \(\operatorname{int}(\Delta^{n-1})\) consists of
vectors with every coordinate positive.

\Needspace{5\baselineskip}
To quantify concentration near a coordinate face, fix an expert index set
\(S\subseteq\{1,\ldots,n\}\) and define its outside-face mass by
\(D_S(p):=\sum_{i\notin S}p_i\).
Small \(D_S(p)\) measures finite-time concentration near the coordinate face
supported on \(S\); \(D_S(p(t))\to0\) means asymptotic concentration on that
face.

Stack the expert outputs \(E_i\in\R^d\) as the rows of
\(E\in\R^{n\times d}\). For a common target \(y\in\R^d\), the routed output is
\(E^\top p\) and the corresponding square loss is
\begin{equation*}
  \mathcal L(p,E)=\|E^\top p-y\|^2.
\end{equation*}
The gradient flows below optimize the half-loss
\(\mathcal H(p,E):=\mathcal L(p,E)/2\); we retain \(\mathcal L\) for loss
levels and dissipation statements. Write \(R:=E^\top p-y\).

We write \(\|\cdot\|\) or \(\|\cdot\|_2\) for the Euclidean vector
norm, \(\|\cdot\|_1\) for the sum of absolute coordinates,
and \(\|\cdot\|_{\rm op}\) and \(\|\cdot\|_{\rm F}\) for the operator
and Frobenius matrix norms. An unsubscripted matrix norm is the operator norm;
\(\langle\cdot,\cdot\rangle\) is the Euclidean inner product.
We use \(O(\cdot)\) for upper bounds up to a constant and
\(\Theta(\cdot)\) or \(\asymp\) for two-sided bounds by positive constants;
the relevant limit and fixed parameters accompany each result.

For logits \(q\in\R^n\), the softmax map is
\(p_i=e^{q_i}/\sum_{k=1}^n e^{q_k}\). Its Jacobian is
\begin{equation*}
  J(p)=\operatorname{diag}(p)-pp^\top .
\end{equation*}
Here \(\operatorname{diag}(p)\) has diagonal entries \(p_1,\ldots,p_n\).
Write \(e_j\) for the \(j\)th coordinate vector. All logarithms are natural. The half-loss gradients are
\[
 g(p,E,y):=\nabla_p\mathcal H=ER,\qquad
 s(p,E,y):=\nabla_q\mathcal H=J(p)g,\qquad
 \nabla_{E_i}\mathcal H=p_iR.
\]

Euclidean gradient flow in logit space gives
\(\dot q=-s\). Transporting the logit velocity through the softmax map gives
\(\dot p=J(p)\dot q\), and hence the fixed-expert routing dynamics
\begin{equation}
\label{eq:pflow}
  \dot p=-M(p)E(E^\top p-y),\qquad M(p):=J(p)^2 .
\end{equation}
One Jacobian factor differentiates the loss with respect to logits; the other transports logit velocity to probability space.

With expert-to-router learning-rate ratio \(\eta\geq0\), training the experts
on the same half-loss gives the coupled dynamics
\begin{equation}
\label{eq:jointgradientflow}
\begin{cases}
\dot p=-M(p)E(E^\top p-y),\\
\dot E_i=-\eta p_i(E^\top p-y),\qquad i=1,\ldots,n .
\end{cases}
\end{equation}
Each \(E_i\) is a freely trainable output vector, and \(\eta\) scales its
update relative to the router. Optimizing the full loss instead of the
half-loss rescales time by a factor of two.

The softmax log-ratio identity is
\begin{equation*}
  \frac{d}{dt}\log\frac{p_i}{p_j}=-(s_i-s_j)
\end{equation*}
for every pair of positive routing coordinates. It relates routing-coordinate evolution to relative logit motion.

\subsection{Loss dissipation and equilibria on simplex faces}

The simplex structure is preserved by the routing dynamics. Indeed,
\(J(p)\mathbf1=0\), and any zero coordinate annihilates the corresponding row
and column of \(J(p)\). Therefore, the simplex and each of its faces are
invariant under the fixed-expert flow. Moreover, for fixed \(E\),
\begin{equation}
\label{eq:fixedflowdissipation}
\frac{d}{dt}\mathcal L(p(t),E)
=-2\|J(p)E(E^\top p-y)\|^2\leq0 .
\end{equation}

The static minimizers represent the projection of \(y\) onto
\(\mathcal C(E):=\operatorname{conv}\{E_1,\ldots,E_n\}\)
\cite{rockafellar1997convex}.
Boundary equilibria can instead represent suboptimal restricted-hull projections.

A point with support \(S=\operatorname{supp}(p):=\{i:p_i>0\}\) is stationary
exactly when \(g_i=E_i^\top(E^\top p-y)\) is constant over \(i\in S\). This is
the first-order stationarity condition for the convex loss restricted to the
corresponding simplex face. Hence every interior stationary point solves the
full projection problem, whereas a boundary stationary point may violate the
inactive-coordinate Karush--Kuhn--Tucker (KKT) inequalities. The next section uses stationary loss levels to identify the selected face.

\Needspace{6\baselineskip}

\section{Face selection with fixed experts}
\label{sec:theory}

For fixed experts, convex projection identifies the optimal routed output.
We use the gap between stationary loss levels to determine the hull face
supporting every limiting representation. The geometric bounds below then
quantify concentration near that face.

For a nonempty closed convex set \(C\), write
\(\operatorname{Proj}_C y\) for Euclidean projection. For a point \(p\)
and a set \(C\) in the same space, write
\(\operatorname{dist}(p,C):=\inf_{v\in C}\|p-v\|\).

\subsection{Stationary-loss barrier and face selection}

For every nonempty support set
\(S\subseteq\{1,\ldots,n\}\), define the convex hull of the expert outputs
indexed by \(S\), its projection of the target, and its corresponding loss
level by
\[
 \mathcal C_S:=\operatorname{conv}\{E_i:i\in S\},\qquad
 z_S:=\operatorname{Proj}_{\mathcal C_S}y,\qquad
 \ell_S:=\|z_S-y\|^2 .
\]

The closed simplex face indexed by \(S\) is
\(\Delta_S:=\{p\in\Delta^{n-1}:p_i=0\text{ for }i\notin S\}\),
with relative interior
\(\Delta_S^\circ:=\{p\in\Delta^{n-1}:\supp(p)=S\}\). Define the
finite set of stationary supports by
\[
 \mathfrak S_{\rm stat}
 :=\{S\subseteq\{1,\ldots,n\}:S\neq\varnothing,\quad
       z_S=E^\top p\text{ for some }p\in\Delta_S^\circ\}.
\]

\Needspace{12\baselineskip}

\begin{theorem}[Face selection below the stationary-loss barrier]
\label{thm:generalfaceselection}
Let \(E\in\R^{n\times d}\), \(n\geq2\), \(y\in\R^d\), and
\(z^\ast:=\operatorname{Proj}_{\mathcal C(E)}y\).
Let \(G^\ast\) be the smallest face of \(\mathcal C(E)\) containing
\(z^\ast\). Define its expert index set \(F^\ast\) and the set
\(\mathcal P^\ast\) of routing vectors representing \(z^\ast\) by
\[
 F^\ast:=\{i:E_i\in G^\ast\},\qquad
 \mathcal P^\ast:=\{p\in\Delta^{n-1}:E^\top p=z^\ast\},
\]
and let \(p(t)\) solve
\eqref{eq:pflow} from an interior initial point. Define the suboptimal
stationary supports and their lowest loss by
\[
\mathfrak S_{\rm comp}
:=\{S\in\mathfrak S_{\rm stat}:z_S\neq z^\ast\},\qquad
\ell_{\rm comp}
:=\min_{S\in\mathfrak S_{\rm comp}}\ell_S,
\]
with the convention \(\ell_{\rm comp}=+\infty\) when
\(\mathfrak S_{\rm comp}=\varnothing\). Assume that the trajectory enters
the strict sublevel below this level: for some finite \(t_{\rm ent}\geq0\),
\begin{equation}
\label{eq:faceenergyentry}
  \mathcal L(p(t_{\rm ent}),E)<\ell_{\rm comp}.
\end{equation}
Then, as \(t\to\infty\),
\[
 \operatorname{dist}\bigl(p(t),\mathcal P^\ast\bigr)\to0,\qquad
 E^\top p(t)\to z^\ast .
\]
These conclusions include \(G^\ast=\mathcal C(E)\). If \(G^\ast\) is proper,
then the mass outside its expert index set satisfies
\[
 D_{F^\ast}(p(t))\to0 .
\]
For the full face, \(D_{F^\ast}\equiv0\).
If \(\mathcal P^\ast=\{p^\ast\}\), then \(p(t)\to p^\ast\).
\end{theorem}

Routing approaches the coordinate face \(\Delta_{F^\ast}\), but the
representation set \(\mathcal P^\ast\) may contain multiple vectors.
Selection therefore need not give positive limiting mass to every expert
in \(F^\ast\).

\begin{proof}[Proof of \Cref{thm:generalfaceselection}]
Set \(c=\mathcal L(p(t_{\rm ent}),E)<\ell_{\rm comp}\).
The closed simplex sublevel set
\(K_c=\{p\in\Delta^{n-1}:\mathcal L(p,E)\leq c\}\)
is compact and positively invariant. By
\eqref{eq:fixedflowdissipation}, zero dissipation is equivalent to
\(J(p)E(E^\top p-y)=0\),
which also makes the vector field vanish. LaSalle's principle
\cite{hofbauer1998evolutionary,sandholm2010population}
therefore implies that the \(\omega\)-limit set is contained in the set of
equilibria.

For an equilibrium with support \(S\), the active components \(g_i\),
\(i\in S\), are identical. The convexity of the square loss restricted to the
face \(\Delta_S\) then identifies the corresponding output as \(z_S\). Hence
\(S\in\mathfrak S_{\rm stat}\), and membership in \(K_c\) excludes every
stationary output satisfying \(z_S\neq z^\ast\). Consequently, every limit
point belongs to \(\mathcal P^\ast\). Compactness then yields convergence in
distance to \(\mathcal P^\ast\) and convergence of the routed output to
\(z^\ast\), including the case where the projection lies in the relative
interior of the full convex hull.

When \(G^\ast\) is proper, the face property implies that every convex
representation of \(z^\ast\) assigns zero weight to experts outside
\(F^\ast\). Continuity of the trajectory therefore gives
\(D_{F^\ast}(p(t))\to0\). If the representation set is a singleton,
the distance convergence reduces to pointwise convergence.
\end{proof}

\subsection{Quantitative bounds on face concentration}

\begin{corollary}[A neighborhood below the stationary-loss barrier]
\label{cor:face-neighborhood}
Retain the notation of \Cref{thm:generalfaceselection} and suppose
\(\mathfrak S_{\rm comp}\neq\varnothing\). Set
\[
 r_\ast:=\|z^\ast-y\|,
 \qquad \Delta_{\rm face}:=\ell_{\rm comp}-r_\ast^2>0,
 \qquad d_\ast(p):=\operatorname{dist}(p,\mathcal P^\ast).
\]
Every \(p\in\Delta^{n-1}\) satisfying
\begin{equation*}
 \|E\|_{\rm op}d_\ast(p)
 <\sqrt{r_\ast^2+\Delta_{\rm face}}-r_\ast
\end{equation*}
lies below the lowest suboptimal stationary loss. Thus every interior
initialization in this neighborhood satisfies
\eqref{eq:faceenergyentry} with \(t_{\rm ent}=0\).
\end{corollary}

\begin{proof}
Let \(p^\ast\) attain \(d_\ast(p)\) and define
\(u=\|E^\top(p-p^\ast)\|\leq \|E\|_{\rm op}d_\ast(p)\).
The triangle inequality gives
\[
\mathcal L(p,E)-r_\ast^2\leq2r_\ast u+u^2,
\]
which proves the claim.
\end{proof}

\begin{corollary}[Loss excess controls mass outside the optimal face]
\label{cor:exposed-face-mass}
Retain the notation of \Cref{thm:generalfaceselection} and assume
\(G^\ast\) is proper. Choose an exposing vector \(v\in\R^d\) and level
\(b\in\R\) for the polytope face \(G^\ast\), oriented so that
\[
 v^\top z=b\quad(z\in G^\ast),\qquad
 v^\top E_i>b\quad(i\notin F^\ast),
\]
and set
\(\kappa_F:=\min_{i\notin F^\ast}(v^\top E_i-b)>0\).
Then every \(p\in\Delta^{n-1}\) satisfies
\begin{equation*}
 D_{F^\ast}(p)
 \leq \frac{\|v\|}{\kappa_F}\|E^\top p-z^\ast\|
 \leq \frac{\|v\|}{\kappa_F}
 \sqrt{\mathcal L(p,E)-\|z^\ast-y\|^2}.
\end{equation*}
The factor \(\|v\|/\kappa_F\) measures face separation in the chosen direction and is invariant under positive rescaling of \((v,b)\).
\end{corollary}

\begin{proof}
For \(z=E^\top p\), exposedness gives
\[
\kappa_FD_{F^\ast}(p)
\leq v^\top(z-z^\ast)
\leq\|v\|\|z-z^\ast\|.
\]
The projection inequality
\(\langle y-z^\ast,z-z^\ast\rangle\leq0\) yields
\[
\mathcal L(p,E)-\|z^\ast-y\|^2\geq\|z-z^\ast\|^2,
\]
which proves the claim.
\end{proof}

The bound in \Cref{cor:exposed-face-mass} uses loss excess and an exposing direction to control outside-face mass. The neighborhood in \Cref{cor:face-neighborhood} is determined by the stationary-loss gap.

\begin{proposition}[Finite-time entry into an optimal-face neighborhood]
\label{prop:outside-face-hitting}
Retain the notation of
\Cref{thm:generalfaceselection,cor:face-neighborhood}, assume
\(G^\ast\) is proper, and choose
\(r_\ast^2<c<\ell_{\rm comp}\) and \(\delta>0\). Define
\[
 \mathcal K_{c,\delta}
 :=\{p\in\Delta^{n-1}:\mathcal L(p,E)\leq c,
          \ D_{F^\ast}(p)\geq\delta\},
\]
and
\[
 \mu(c,\delta)
 :=\min_{p\in\mathcal K_{c,\delta}}
 \|J(p)E(E^\top p-y)\|^2 ,
\]
whenever \(\mathcal K_{c,\delta}\neq\varnothing\). Then
\(\mu(c,\delta)>0\). If a fixed-expert trajectory satisfies
\(\mathcal L(p(t_0),E)\leq c\) at a time \(t_0\geq0\) and
\(D_{F^\ast}(p(t_0))\geq\delta\), its first entry time
\(\tau_\delta:=\inf\{t\geq t_0:D_{F^\ast}(p(t))<\delta\}\)
obeys
\begin{equation}
\label{eq:outside-face-hitting-time}
 \tau_\delta-t_0
 \leq
 \frac{\mathcal L(p(t_0),E)-r_\ast^2}
 {2\mu(c,\delta)} .
\end{equation}
\end{proposition}

\begin{proof}
Compactness gives the minimum. If the minimum were zero, there would exist a
stationary point below \(\ell_{\rm comp}\) with
\(D_{F^\ast}\geq\delta\). Such a point would necessarily represent the
optimal output \(z^\ast\), contradicting the fact that every representation
of \(z^\ast\) has zero mass outside \(F^\ast\). Hence
\(\mu(c,\delta)>0\). Integrating
\(\frac{d}{dt}\mathcal L(p(t),E)\leq-2\mu(c,\delta)\)
until the first hitting time gives
\eqref{eq:outside-face-hitting-time}.
\end{proof}

A lower bound on \(\mu(c,\delta)\) makes the entry-time estimate explicit.
Appendices~\ref{app:subsection-1-1}--\ref{app:subsection-1-2} and \ref{sec:sm-face-selection-detail} give geometric bounds and
initialization results.

\section{Expert adaptation: transient concentration and interior convergence}
\label{sec:coupled-main}

Slow expert adaptation preserves fixed-expert trajectories on finite
horizons. For every fixed positive rate, however, the common-target joint
square-loss flow has an interior long-time limit.

\subsection{Finite-horizon approximation by fixed-expert dynamics}

Let \(p_0(t)\) denote the fixed-expert solution of \eqref{eq:pflow} with
expert matrix \(E^0\), and let \((p_\eta(t),E_\eta(t))\) denote the joint-flow
solution of \eqref{eq:jointgradientflow} from the same router state and
\(E_\eta(0)=E^0\).

\begin{theorem}[Finite-horizon slow-expert approximation]
\label{thm:fastslowbridge}
Fix \(T<\infty\), finite \(E^0\), and an interior initial router state. There
exist \(\eta_T>0\) and \(C_T<\infty\), depending on the initial state, target,
and \(T\), such that for every \(0\leq\eta\leq\eta_T\),
\begin{equation}
\label{eq:fastslowbound}
 \sup_{0\leq t\leq T}
 \left(
 \|p_\eta(t)-p_0(t)\|
 +\|E_\eta(t)-E^0\|_{\mathrm F}
 \right)
 \leq C_T\eta .
\end{equation}
\end{theorem}

\begin{proof}
Write \(R_\eta=E_\eta^\top p_\eta-y\) and
\(R_0=E^{0\top}p_0(0)-y\). Differentiating the joint loss gives
\[
 \dot{\mathcal L}=-2\|J(p_\eta)E_\eta R_\eta\|^2
 -2\eta\|p_\eta\|^2\|R_\eta\|^2\leq0.
\]
Consequently, \(\|\dot E_\eta\|_{\rm F}\leq\eta\|R_0\|\), so
\(\sup_{t\leq T}\|E_\eta(t)-E^0\|_{\rm F}\leq\eta T\|R_0\|\).
For \(0\leq\eta\leq1\), these matrices remain in a common bounded ball.
The polynomial router field
\(f(p,E)=-J(p)^2E(E^\top p-y)\) has finite Lipschitz constants
\(L_T,K_T\) there, uniformly for \(p\) in the simplex. These bounds also
prevent finite-time escape. With
\(d(t)=\|p_\eta(t)-p_0(t)\|\), the integral equations give
\[
 d(t)\leq\int_0^t\bigl(L_Td(s)+K_T\eta s\|R_0\|\bigr)\,ds
 \leq \tfrac12 K_T\eta\|R_0\|t^2e^{L_Tt},
\]
where the final inequality is Gronwall's estimate. Combining the bounds
proves \eqref{eq:fastslowbound}.
\end{proof}

Appendix~\ref{app:subsection-1-3} gives explicit constants for this finite-horizon bound.

\subsection{Interior convergence of the joint square-loss flow}

The square-loss dissipation mechanism in value--softmax flow
\cite{varre2026gradient} gives exponential residual decay. We use this decay
to bound expert motion and total logit motion, obtaining convergence and
explicit positive routing bounds for arbitrary finite initialization.

\begin{theorem}[Interior convergence of the joint square-loss flow]
\label{thm:jointgradientinteriorfreeze}
Let \((q(t),E(t))\) be the maximal solution of the joint gradient flow whose
probability-space form is \eqref{eq:jointgradientflow}, with \(\eta>0\),
finite \(q(0)\), and finite \(E(0)\). Set
\(p(t)=\operatorname{Softmax}(q(t))\) and
\(R(t)=E(t)^\top p(t)-y\), and define
\[
 \rho_0:=\|R(0)\|,\qquad
 B_0:=\|E(0)\|_{\mathrm{op}}+n\rho_0,\qquad
 C_0:=\frac{nB_0\rho_0}{\eta}.
\]
Then the solution exists for all \(t\geq0\) and
\begin{equation}
\label{eq:jointgradientresidualrate}
 \|R(t)\|^2\leq \|R(0)\|^2e^{-2\eta t/n},
 \qquad
 \int_t^\infty\|R(\tau)\|\,d\tau
 \leq \frac{n\rho_0}{\eta}e^{-\eta t/n}.
\end{equation}
Moreover, \(q(t)\to q_\infty\in\R^n\) and
\(E(t)\to E_\infty\), and hence
\begin{equation}
\label{eq:jointgradienttails}
 \|E(t)-E_\infty\|_{\mathrm F}
 \leq n\rho_0e^{-\eta t/n},\qquad
 \|q(t)-q_\infty\|
 \leq C_0e^{-\eta t/n}.
\end{equation}
For every expert \(i\) and every \(t\geq0\),
\begin{equation}
\label{eq:jointgradientfloor}
 p_i(t)\geq
 \frac{p_i(0)}
 {p_i(0)+e^{2C_0}(1-p_i(0))}>0 .
\end{equation}
In particular,
\[
 p(t)\to p_\infty=\operatorname{Softmax}(q_\infty)
 \in\operatorname{int}(\Delta^{n-1}),
 \qquad E(t)^\top p(t)\to y .
\]
Thus, for every finite initialization and \(n\geq2\), the reduced joint
square-loss flow converges to a strictly interior routing vector.
\end{theorem}

\begin{proof}
For the scalar energy
\(\mathcal H(t):=\|R(t)\|^2/2\),
we have
\[
  \dot{\mathcal H}
  =-\|J(p)ER\|^2-\eta\|p\|^2\|R\|^2
  \leq-\frac{2\eta}{n}\mathcal H,
\]
where we used
\(\|p\|^2\geq1/n\) on the simplex. Therefore, the residual decays
exponentially on every interval of existence. Since
\(\dot E=-\eta pR^\top\) and \(\|p\|\leq1\),
\[
 \|E(t)\|_{\rm op}
 \leq\|E(0)\|_{\rm op}
 +\eta\int_0^t\|R(\tau)\|\,d\tau
 \leq B_0 .
\]
The uniform bound on \(E(t)\), together with
\(\|J(p)\|_{\rm op}\leq1\), bounds the logit velocity
\(\dot q=-J(p)ER\), ruling out finite-time escape and proving global
existence. Integrating the residual estimate yields
\eqref{eq:jointgradientresidualrate} and
\[
 \int_t^\infty\|\dot E(\tau)\|_{\rm F}\,d\tau
 \leq n\rho_0e^{-\eta t/n},
 \qquad
 \int_t^\infty\|\dot q(\tau)\|\,d\tau
 \leq C_0e^{-\eta t/n},
\]
which proves the convergence of \(E(t)\), \(q(t)\), and
\eqref{eq:jointgradienttails}. The bound
\(\|q(t)-q(0)\|\leq C_0\), inserted into the softmax ratios, gives
\eqref{eq:jointgradientfloor}. The remaining conclusions follow from
\(R(t)\to0\).
\end{proof}

Appendix~\ref{app:subsection-1-4} extends the argument to residual losses with matching Polyak--\L{}ojasiewicz lower and gradient upper bounds.

\subsection{Transient face concentration under slow expert adaptation}

\begin{corollary}[Transient face concentration with an interior limit]
\label{cor:transient-face-fixed-rate-interior}
Let \(E^0,y\), and the common interior initial router state satisfy the
hypotheses of \Cref{thm:generalfaceselection}, with proper selected index set
\(F^\ast\). Use \(D_{F^\ast}\) from that theorem and let \(p_\eta\) denote
the joint flow from \(E^0\). For every \(\delta\in(0,1)\), there exist
\(T_\delta<\infty\) and \(\eta_\delta>0\) such that, for every
\(\eta\in[0,\eta_\delta]\),
\begin{equation}
\label{eq:transient-face-approach}
 D_{F^\ast}(p_\eta(T_\delta))<\delta .
\end{equation}
For every fixed \(\eta>0\), there exists
\(p_{\eta,\infty}\in\operatorname{int}(\Delta^{n-1})\) such that
\(p_\eta(t)\to p_{\eta,\infty}\) and
\begin{equation}
\label{eq:fixed-rate-interior-off-face}
 \lim_{t\to\infty}D_{F^\ast}(p_\eta(t))
 =\sum_{i\notin F^\ast}(p_{\eta,\infty})_i>0 .
\end{equation}
Thus arbitrarily slow expert learning preserves the finite-time approach to
the selected simplex face, whereas every fixed positive expert-update rate
produces a strictly interior long-time limit in this model.
\end{corollary}

\begin{proof}
Choose \(T_\delta\) such that the fixed-expert solution satisfies
\(D_{F^\ast}(p_0(T_\delta))<\delta/2\).
The finite-horizon estimate
\eqref{eq:fastslowbound} and the continuity bound
\(|D_{F^\ast}(p)-D_{F^\ast}(\tilde p)|\leq\|p-\tilde p\|_1\)
imply
\eqref{eq:transient-face-approach} for sufficiently small \(\eta\). For
fixed \(\eta>0\), \Cref{thm:jointgradientinteriorfreeze} gives convergence to
an interior limiting routing vector. Since \(F^\ast\) is a proper subset, the
complementary routing mass remains strictly positive, yielding
\eqref{eq:fixed-rate-interior-off-face}.
\end{proof}

The interior mass floor need not be uniform as \(\eta\downarrow0\).
Recovery to a prescribed threshold requires sufficient normal motion,
which we study next.

\section{Recovery-time scaling near simplex boundaries}
\label{sec:boundary-recovery}

Near a simplex face, complementary-expert covariance and relative gradient advantage determine the normal velocity. We separate these contributions and relate the resulting recovery coefficient to the time needed to reach a threshold.

\subsection{Normal-velocity decomposition}

\begin{proposition}[Exact normal-velocity decomposition]
\label{prop:generalfaceboundarymobility}
Let \(p(t)\in\operatorname{int}(\Delta^{n-1})\) solve
\(\dot p=-J(p)^2g(t)\), with \(n\geq2\),
where \(g(t)\in\R^n\) is locally integrable along the trajectory. Fix an expert
\(j\), write \(x=p_j\), and let \(w=(w_k)_{k\neq j}\) be the conditional
probability vector on the face \(p_j=0\), with
\[
 w_k:=\frac{p_k}{1-x},\qquad
 \bar g_{-j}:=\sum_{k\neq j}w_kg_k,\qquad
 \Delta_j:=\bar g_{-j}-g_j .
\]
A positive \(\Delta_j\) favors the suppressed coordinate relative to the
conditional mean gradient of its complementary experts. Define
\begin{equation*}
 \Lambda_j:=\sum_{k\neq j}w_k^2(g_k-\bar g_{-j}),
 \qquad
 \Gamma_j:=\Lambda_j+x(1+\|w\|_2^2)\Delta_j,
\end{equation*}
where \(\Lambda_j=\operatorname{Cov}_{k\sim w}(w_k,g_k)\) is the weighted
covariance for an index \(k\) sampled according to \(w_k\), and
\(\Gamma_j\) is the effective recovery coefficient. Positive covariance contributes to recovery even when \(\Delta_j=0\). The following identities
hold almost everywhere:
\begin{equation}
\label{eq:generalfaceexactnormal}
\begin{aligned}
 \dot x&=x(1-x)^2\bigl[\Lambda_j+x(1+\|w\|_2^2)\Delta_j\bigr]
 =x(1-x)^2\Gamma_j,\\
 \frac{d}{dt}\log\frac{x}{1-x}&=(1-x)\Gamma_j .
\end{aligned}
\end{equation}
\end{proposition}

\begin{proof}
Let \(s=J(p)g\). Differentiating the singleton log-odds and substituting the
conditional probabilities gives
\[
 \frac{d}{dt}\log\frac{x}{1-x}
 =\frac{\sum_{k\neq j}p_ks_k}{1-x}-s_j
 =(1-x)\bigl[\Lambda_j+x(1+\|w\|_2^2)\Delta_j\bigr].
\]
Multiplication by \(x(1-x)\) proves the normal-velocity identity.
\end{proof}

The covariance term is positive when larger complementary weights \(w_k\) tend to accompany larger probability-gradient values \(g_k\). It can therefore increase \(p_j\) even when the mean relative advantage \(\Delta_j\) is zero. When the two terms oppose each other, their sum \(\Gamma_j\) determines whether \(p_j\) increases.

\subsection{Boundary degeneracy and recovery-time scaling}

The factor \(x\) in \eqref{eq:generalfaceexactnormal} expresses softmax
boundary degeneracy. The normal velocity is \(\Theta(x^{1+\alpha})\) when
\(\Gamma_j\asymp x^\alpha\), including linear motion when
\(\Gamma_j\asymp1\). The recovery-time bounds below compare the coefficient
with powers of \(x\) throughout the path to the threshold.

\begin{theorem}[Recovery time from the order of normal degeneracy]
\label{thm:mobility-exponent}
Retain the notation and hypotheses of
\Cref{prop:generalfaceboundarymobility}, with \(x(0)=\epsilon\) and
\(0<\epsilon<\theta<1\). Define
\(T_\theta:=\inf\{t\geq0:x(t)=\theta\}\), with \(\inf\varnothing=\infty\).
Assume the solution exists and remains in the simplex interior on
\([0,T_\theta)\), with global existence when \(T_\theta=\infty\).
For \(\alpha\geq0\), put
\[
 H_{1+\alpha}(a,b):=\int_a^b\frac{du}{u^{1+\alpha}(1-u)^2}.
\]
If there are constants \(0<c\leq c'\) with
\[
 cx(t)^\alpha\leq\Gamma_j(t)\leq c'x(t)^\alpha,
 \qquad 0\leq t<T_\theta,
\]
then \(T_\theta<\infty\) and
\begin{equation}
\label{eq:mobility-exponent-bounds}
 \frac{H_{1+\alpha}(\epsilon,\theta)}{c'}
 \leq T_\theta\leq
 \frac{H_{1+\alpha}(\epsilon,\theta)}{c}.
\end{equation}
For a family indexed by \(\epsilon\), with fixed \(n,\theta,\alpha\)
and constants \(c,c'\) independent of \(\epsilon\), this gives
\(T_\theta=\Theta(\log(1/\epsilon))\) when \(\alpha=0\) and
\(T_\theta=\Theta(\epsilon^{-\alpha})\) when \(\alpha>0\).
\end{theorem}

\begin{proof}
Substituting the exponent bounds into
\eqref{eq:generalfaceexactnormal} gives
\[
cx^{1+\alpha}(1-x)^2
\leq\dot x
\leq
c'x^{1+\alpha}(1-x)^2,
\]
and separation of variables yields
\eqref{eq:mobility-exponent-bounds}. For \(\alpha>0\), the uniform bounds
\(1\leq(1-u)^{-2}\leq(1-\theta)^{-2}\)
on \([\epsilon,\theta]\) yield
\[
 \frac{\epsilon^{-\alpha}-\theta^{-\alpha}}{\alpha}
 \leq H_{1+\alpha}(\epsilon,\theta)
 \leq
 \frac{\epsilon^{-\alpha}-\theta^{-\alpha}}
 {\alpha(1-\theta)^2}.
\]
This proves the polynomial order for every fixed \(\alpha>0\). For
\(\alpha=0\), \(H_1(\epsilon,\theta)=\log(1/\epsilon)+O_\theta(1)\),
where the subscript allows dependence on the fixed threshold.
\end{proof}

A coefficient bounded away from zero gives logarithmic recovery; linear
vanishing gives inverse-probability recovery. The decomposition reduces the
recovery-time question to the evolution of complementary-expert covariance
and relative gradient advantage along the path. The next corollary gives
sufficient structural conditions for these two regimes.

\begin{corollary}[Logarithmic and inverse-probability recovery criteria]
\label{cor:structural-recovery}
Retain the trajectory and recovery-time hypotheses of
\Cref{thm:mobility-exponent}, with fixed \(n,\theta\). Write \(H_1\) and
\(H_2\) for the corresponding integrals at \(\alpha=0\) and \(\alpha=1\).
All bounds below hold until \(T_\theta\), uniformly over the small-mass family.
\begin{enumerate}
\item If \(0<\lambda_-\leq\Lambda_j\leq\lambda_+\),
\(|\Delta_j|\leq D\), and \(2\theta D<\lambda_-\), then
\[
 \frac{H_1(\epsilon,\theta)}{\lambda_++2\theta D}
 \leq T_\theta\leq
 \frac{H_1(\epsilon,\theta)}{\lambda_--2\theta D},
 \qquad T_\theta=\Theta(\log(1/\epsilon)).
\]
\item If \(|\Lambda_j|\leq Cx\),
\(0<\Delta_-\leq\Delta_j\leq\Delta_+\), and
\(c_-:=(1+1/(n-1))\Delta_- - C>0\), then
\[
 \frac{H_2(\epsilon,\theta)}{C+2\Delta_+}
 \leq T_\theta\leq
 \frac{H_2(\epsilon,\theta)}{c_-},
 \qquad T_\theta=\Theta(\epsilon^{-1}).
\]
\end{enumerate}
In particular, positive bounds
\(\underline\Gamma\leq\Gamma_j\leq\overline\Gamma\) give
\begin{equation*}
 \frac{H_1(\epsilon,\theta)}{\overline\Gamma}
 \leq T_\theta\leq
 \frac{H_1(\epsilon,\theta)}{\underline\Gamma}.
\end{equation*}
When \(\Lambda_j=0\) and
\(0<\underline\Delta\leq\Delta_j\leq\overline\Delta\), the second case
specializes to
\begin{equation*}
 \frac{H_2(\epsilon,\theta)}{2\overline\Delta}
 \leq T_\theta\leq
 \frac{H_2(\epsilon,\theta)}
 {(1+1/(n-1))\underline\Delta}.
\end{equation*}
\end{corollary}

\begin{proof}
Since \(\frac1{n-1}\leq\|w\|_2^2\leq1\),
the first case bounds \(\Gamma_j\) between
\(\lambda_--2\theta D\) and \(\lambda_++2\theta D\). The second case bounds
\(\Gamma_j\) between \(c_-x\) and \((C+2\Delta_+)x\). Applying the recovery
estimate
\(\dot x=x(1-x)^2\Gamma_j\)
and integrating gives the claims.
\end{proof}

\subsection{Optimization geometry and recovery scaling}

For two experts, \(\Lambda_j=0\). An advantage bounded above and away from zero yields inverse-probability recovery under squared-Jacobian flow. Comparing three flows under the same advantage signal isolates the effect of optimization geometry.

\begin{corollary}[Two-expert recovery times]
\label{cor:two-expert-recovery-clocks}
For two experts, write \(p(t)\) for the scalar probability of the expert
favored after the target changes, and let \(0<p(0)=\epsilon<\theta<1\).
Write \(g_{\rm fav}\) and \(g_{\rm other}\) for the corresponding components
of \(g\), and put
\(\Delta g=g_{\rm other}-g_{\rm fav}\). In the following three flows,
\(J\) is evaluated at \((p,1-p)\), and \(I\) is the \(2\times2\) identity
matrix:
\[
 \dot p=
 \begin{cases}
 2p^2(1-p)^2\Delta g(t),&M(p)=J(p)^2,\\
 p(1-p)\Delta g(t),&M(p)=J(p),\\
 \frac12\Delta g(t),&M(p)=I-\frac12\mathbf1\mathbf1^\top ,
 \end{cases}
\]
and, for \(0<a,b<1\), define
\[
 I_2(a,b):=\int_a^b\frac{du}{2u^2(1-u)^2},\quad
 I_1(a,b):=\log\frac{b(1-a)}{(1-b)a},\quad
 I_0(a,b):=2(b-a),
\]
with \(A(t)=\int_0^t\Delta g(\tau)\,d\tau\) and \(r=2,1,0\) in the three
cases, respectively. Until first boundary contact, the corresponding flow
satisfies
\begin{equation*}
 I_r(\epsilon,p(t))=A(t).
\end{equation*}
Hence, on any horizon \(T\) preceding boundary contact, recovery by \(T\) is
equivalent to
\[
\sup_{t\leq T}A(t)\geq I_r(\epsilon,\theta).
\]
Let \(T_\theta\) denote the first hitting time. If, for
\(0\leq t<T_\theta\),
\(0<\underline\gamma\leq\Delta g\leq\overline\gamma\), then
\[
 \frac{I_r(\epsilon,\theta)}{\overline\gamma}
 \leq T_\theta\leq
 \frac{I_r(\epsilon,\theta)}{\underline\gamma}.
\]
For fixed \(\theta\),
\[
I_2(\epsilon,\theta)=\frac1{2\epsilon}
+O\!\left(\log\frac1\epsilon\right),\qquad
I_1(\epsilon,\theta)=\log\frac1\epsilon+O_\theta(1),\qquad
I_0(\epsilon,\theta)=2\theta+O(\epsilon).
\]
Thus advantages bounded above and away from zero, uniformly over the initial-mass family, give
\(\Theta(\epsilon^{-1})\), \(\Theta(\log(1/\epsilon))\), and
\(\Theta(1)\) recovery times, respectively.
\end{corollary}

\begin{proof}
The chain rule gives \(\frac{d}{dt}I_r(\epsilon,p(t))=\Delta g(t)\).
Integration, monotonicity, and the endpoint expansions give the claims.
\end{proof}

\section{Input-conditioned routing and discrete recovery}
\label{sec:neural-bridge}

For a neural router, each input's probabilities depend on shared parameters
and changing representations. We characterize initial separation under an
affine router, then derive an exact log-odds criterion for discrete updates.

\subsection{First-order separation under a shared affine router}

\begin{proposition}[First-order separation of affine routing distributions]
\label{prop:coupledbreakfreezefirstorder}
Let \(K\) groups have weights \(\pi_k>0\) with \(\sum_k \pi_k=1\), targets
\(y_k\in\R^d\), and fixed feature vectors \(h_k\in\R^{d_h}\),
indexed by \(k=1,\ldots,K\). Consider shared parameters
\(W\in\R^{n\times d_h}\) and \(b\in\R^n\) in the affine router
\[
 q^{(k)}=Wh_k+b,\qquad p^{(k)}=\operatorname{Softmax}(q^{(k)}),
\]
trained by gradient flow on
\[
 \Phi(W,b,E)
 :=\frac12\sum_{k=1}^K \pi_k\|E^\top p^{(k)}-y_k\|^2.
\]
The experts may be fixed or may follow their joint gradient flow with
expert-rate multiplier \(\eta_E\geq0\), with \(\eta_E=0\) corresponding to fixed experts. Suppose that at a state
\((W_0,b_0,E_0)\) all routing distributions coincide at \(p_0\). Define
\begin{equation*}
 \mathcal A_0
 :=\sum_{a=1}^K \pi_a
 M(p_0)E_0(E_0^\top p_0-y_a)h_a^\top\in\R^{n\times d_h}.
\end{equation*}
Then, for every pair of groups \(k,\ell\),
\begin{equation}
\label{eq:featureconditionedvelocitydifference}
 \dot p^{(k)}(0)-\dot p^{(\ell)}(0)
 =-\mathcal A_0(h_k-h_\ell).
\end{equation}
Consequently, if
\[
 \chi_{\rm feat}:=\max_{k,\ell}
 \|\mathcal A_0(h_k-h_\ell)\|>0,
\]
at least two group routing distributions differ for all sufficiently small
positive times, with
\[
 \max_{k,\ell}\|p^{(k)}(t)-p^{(\ell)}(t)\|
 =\chi_{\rm feat}t+o(t)\qquad(t\downarrow0).
\]
Here \(o(t)\) denotes a remainder whose ratio to \(t\) tends to zero.
If all features coincide, or all feature contrasts lie in
the nullspace \(\ker(\mathcal A_0)\), first-order separation is impossible.
With equal weights and one-hot features, the criterion reduces, up to a
common time factor, to the independent-group case.
\end{proposition}

\begin{proof}
Let \(R_a=E_0^\top p_0-y_a\). At the shared point,
\[
 \dot p^{(k)}
 =-\sum_{a=1}^K \pi_aM(p_0)E_0R_a(1+h_a^\top h_k).
\]
Subtracting the expressions for \(k\) and \(\ell\) yields
\eqref{eq:featureconditionedvelocitydifference} and its first-order
expansion. The instantaneous router velocity depends on the current expert
outputs, while expert velocities influence subsequent evolution.
\end{proof}

The action of \(\mathcal A_0\) on a feature contrast determines initial
routing separation; shared-bias updates cancel from this contrast.
Appendix~\ref{sec:shared-conditioned} gives the shared-bias identity and tests the
affine separation criterion.

\subsection{Log-odds criterion for discrete recovery}

Accumulated relative-logit motion characterizes threshold recovery in both continuous and discrete time. For neural routers, the identity applies separately to each input.

For \(\dot p=-J(p)^2g\), set \(s=J(p)g\). Integrating the normal log-odds
velocity gives
\begin{equation*}
 \log\frac{p_j(T)}{1-p_j(T)}
 =\log\frac{p_j(0)}{1-p_j(0)}
 +\int_0^T a_j(t)\,dt,
 \qquad
 a_j:=\frac{\sum_{k\neq j}p_ks_k}{1-p_j}-s_j .
\end{equation*}
Thus a coordinate starting at \(\epsilon\) reaches \(\theta\) precisely when
the accumulated log-odds change reaches
\(B_{\theta,\epsilon}:=\log\frac{\theta(1-\epsilon)}{(1-\theta)\epsilon}\).
The discrete-time counterpart holds for arbitrary optimizer increments.

\begin{proposition}[Exact discrete-time log-odds identity]
\label{prop:discrete-logodds-diagnostic}
For integer update indices \(m\geq0\), let \(q^m\in\R^n\),
\(p^m=\operatorname{Softmax}(q^m)\), and
\(u^m=q^{m+1}-q^m\), with arbitrary logit increments.
For a fixed coordinate \(j\), set
\[
 w_k^m:=\frac{p_k^m}{1-p_j^m}\quad(k\neq j),\qquad
 c_j^m:=\log\!\left(\sum_{k\neq j}w_k^m e^{u_k^m}\right).
\]
With \(\operatorname{logit}(u):=\log(u/(1-u))\) for \(u\in(0,1)\),
every integer update horizon \(N\geq1\) satisfies
\begin{equation}
\label{eq:discrete-singleton-budget}
 \operatorname{logit}(p_j^N)
 =\operatorname{logit}(p_j^0)
  +\sum_{m=0}^{N-1}(u_j^m-c_j^m).
\end{equation}
Consequently, if \(p_j^0=\epsilon\) and \(0<\epsilon<\theta<1\), recovery by
step \(N\) occurs if and only if
\begin{equation}
\label{eq:discrete-recovery-crossing}
 \max_{1\leq m\leq N}
 \sum_{\ell=0}^{m-1}(u_j^\ell-c_j^\ell)
 \geq B_{\theta,\epsilon}.
\end{equation}
\end{proposition}

\begin{proof}
Since \(\operatorname{logit}(p_j^m)=q_j^m-\log\sum_{k\neq j}e^{q_k^m}\)
and
\[
c_j^m=\log\sum_{k\neq j}e^{q_k^{m+1}}
       -\log\sum_{k\neq j}e^{q_k^m},
\]
one update contributes \(u_j^m-c_j^m\). Summation and monotonicity give
\eqref{eq:discrete-singleton-budget}--\eqref{eq:discrete-recovery-crossing}.
\end{proof}

Each increment \(u_j^m-c_j^m\) measures target-logit motion relative to all
complementary logits; the maximum cumulative sum accounts for sign-changing
updates. For the continuous flow, the corresponding rate is
\(a_j=(1-p_j)\Gamma_j\). This relates accumulated log-odds to the recovery
coefficient in \Cref{thm:mobility-exponent}.

\Needspace{6\baselineskip}
\subsection{Affine updates and mean routing probability}

For an affine neural router evaluated on an input \(x\), let
\(h^m(x)\in\R^{d_h}\) be its feature vector,
\(W^m\in\R^{n\times d_h}\) its weight matrix, and
\(b^m\in\R^n\) its bias, so that
\(q^m(x)=W^mh^m(x)+b^m\).
The exact logit increment can be decomposed into parameter updates and
representation changes:
\begin{equation*}
 u^m(x)=
 \underbrace{(W^{m+1}-W^m)h^m(x)+(b^{m+1}-b^m)}_{\text{router-parameter contribution}}
 +\underbrace{W^{m+1}(h^{m+1}(x)-h^m(x))}_{\text{representation contribution}}.
\end{equation*}
Here \(q_j^m(x)\) is a raw softmax logit, whereas
\(\operatorname{logit}(p_j^m(x))
=q_j^m(x)-\log\sum_{k\neq j}e^{q_k^m(x)}\) is the scalar one-versus-rest
log-odds. For a fixed collection \(\mathcal X\) of active evaluation-token
inputs, define
\(A_j^N(x):=\sum_{m=0}^{N-1}(u_j^m(x)-c_j^m(x))\).
The mean routing probability is then exactly
\begin{equation*}
 \bar p_j^N
 =
 \frac1{|\mathcal X|}
 \sum_{x\in\mathcal X}
 \sigma\!\left(
 \operatorname{logit}(p_j^0(x))+A_j^N(x)
 \right),
 \qquad
 \sigma(z):=\frac1{1+e^{-z}} .
\end{equation*}
The mean routing probability follows by accumulating each input's log-odds
increments, converting to probabilities, and averaging.
For any comparator \(c\neq j\), the pairwise quantity \(q_j-q_c\) compares the
target with one competing expert, whereas \(c_j^m\) aggregates the motion of
all complementary logits in the exact recovery identity. \Cref{sec:experiments} uses these quantities to relate a router-gradient intervention to input-wise log-odds and the mean routing probability.

\section{Experiments}
\label{sec:experiments}

Reduced-flow integrations compare the derived time scales with numerical
hitting times. Paired Transformer-MoE interventions then test whether
removing a class-conditional router-gradient contrast changes recovery
under matched training. We measure both confirmed recovery of the overall
mean probability and accumulated log-odds on fixed tokens.

\subsection{Experimental setup}

Reduced-flow experiments compare trajectories with theoretical
bounds and independent quadrature calculations, using step refinement to
assess numerical accuracy. Neural comparisons use matched initial parameters,
optimizer states, and minibatch orders; random seeds are treated as
independent units for statistical inference. Layer--initial-mass combinations
and repeated evaluations remain grouped within each seed. Appendices~\ref{sec:sm-face-selection-results}--\ref{sec:sm-statistical-checks} provide numerical refinement, training and inference procedures, and additional input-conditioning experiments.

\subsection{Face selection and expert adaptation}

We enumerate stationary supports in four fixed-expert geometries and
compare their loss barriers with the trajectories in
\Cref{tab:face-selection-evidence}. The representative interior trajectories
enter the strict sublevel and approach the predicted face. A cube trajectory
initialized at a suboptimal boundary equilibrium remains stationary,
illustrating why face invariance alone does not determine selection.

\begin{table}[t]
\caption{Fixed-expert face selection.
Barrier gap is \(\ell_{\rm comp}-r_\ast^2\); entry time is the first sample
with \(\mathcal L<\ell_{\rm comp}\). Outside-face mass \(D_{F^\ast}(p)\)
and projection error \(\|E^\top p-z^\ast\|\) are measured at \(T=2000\).
A dash denotes no entry for the boundary control.}
\label{tab:face-selection-evidence}
\centering
\small
\begin{tabular}{@{}lrrrr@{}}
\toprule
Geometry & Barrier gap & Entry time & \shortstack{Outside mass\\\((\times10^{-3})\)} & \shortstack{Projection error\\\((\times10^{-3})\)}\\
\midrule
Triangle & 0.1225 & 5.38 & 1.62 & 1.64\\
Pentagon edge & 0.3360 & 18.88 & 1.08 & 1.22\\
Redundant hexagon face & 0.0400 & 91.78 & 1.53 & 1.59\\
Cube bottom facet & 0.4880 & 26.60 & 2.51 & 2.51\\
Cube boundary control & 0.4880 & --- & 320 & 482\\
\bottomrule
\end{tabular}
\end{table}

\begin{figure}[t]
\centering
\includegraphics[width=\linewidth]{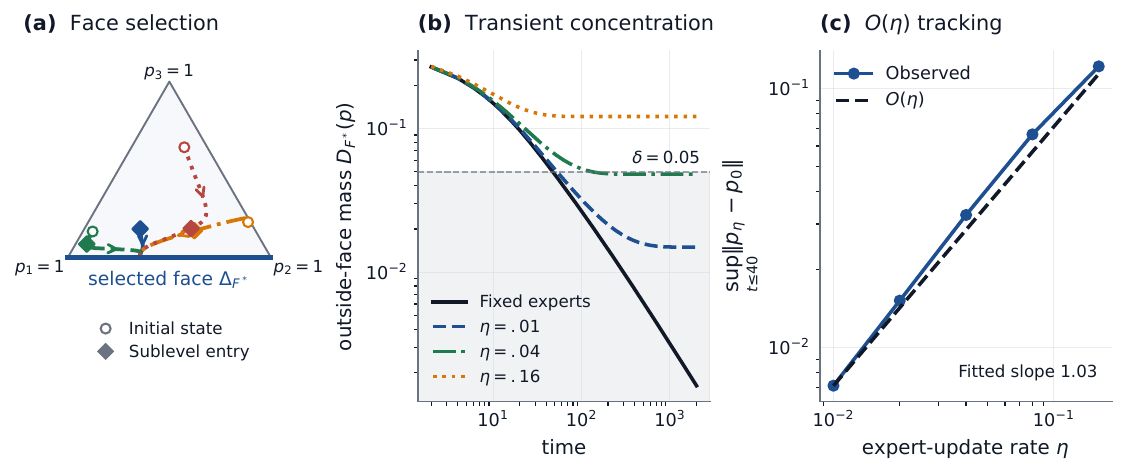}
\caption{Face selection and expert-adaptation time scales.
\textbf{(a)} Triangle trajectories approach \(\Delta_{F^\ast}\); circles mark
initial states and diamonds mark entry below the stationary-loss barrier.
\textbf{(b)} Outside-face mass under fixed and adapting experts; shading
marks \(D_{F^\ast}<0.05\).
\textbf{(c)} Finite-horizon routing error follows the \(O(\eta)\) prediction
of \Cref{thm:fastslowbridge}.}
\label{fig:reduced-flow-overview}
\end{figure}

\Cref{fig:reduced-flow-overview} illustrates the two adaptation time scales.
The finite-horizon routing error scales approximately linearly with
\(\eta\), while positive-rate trajectories retain positive outside-face
mass. This agrees with the approximation and interior-convergence results
of \Cref{thm:fastslowbridge,thm:jointgradientinteriorfreeze}.
Appendices~\ref{sec:sm-face-selection-detail}--\ref{sec:sm-time-scales} and \ref{sec:sm-experiment-sources} give the configurations,
initialization sweep, and numerical refinement.

\subsection{Recovery-time scaling}

We test the covariance and advantage mechanisms in
\Cref{cor:structural-recovery}, then vary the vanishing order of
\(\Gamma_j\) to examine \Cref{thm:mobility-exponent}.

\begin{figure}[t]
\centering
\includegraphics[width=\linewidth]{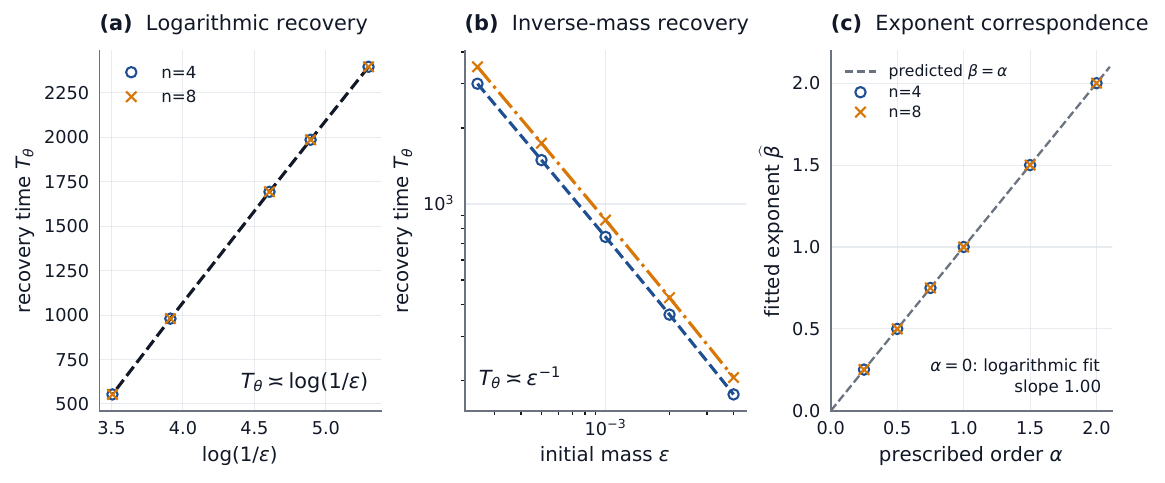}
\caption{Recovery-time scaling in constructed fields.
\textbf{(a)} Fixed positive covariance and zero mean advantage give
logarithmic recovery.
\textbf{(b)} Zero covariance and uniform complementary weights give
inverse-probability recovery under a fixed positive advantage.
Lines show theoretical hitting times and symbols show numerical events
for \(n=4,8\), with \(\theta=0.05\); the curves in (a) coincide.
\textbf{(c)} Fitted exponents agree with prescribed orders in
\(\Gamma_j=x^\alpha\); \(\alpha=0\) uses a separate logarithmic fit.}
\label{fig:multiexpert-recovery-main}
\end{figure}

Numerical hitting times in \Cref{fig:multiexpert-recovery-main}(a,b) agree
with the theoretical integrals to relative error below \(4\times10^{-7}\).
The two fields realize logarithmic and inverse-probability recovery under
the same squared-Jacobian geometry, showing how covariance and advantage
produce different orders of normal motion.

Prescribing \(\Gamma_j=x^\alpha\) extends the comparison to a range of
vanishing orders. The fitted exponents in panel (c) agree with the imposed
orders; at \(\alpha=0\), we fit recovery time against \(\log(1/\epsilon)\).
These constructions test the recovery-time classification across coefficient
orders within squared-Jacobian flow.
Appendices~\ref{sec:sm-recovery-scaling}--\ref{sec:sm-experiment-sources} give the fields, their regularity,
fitting windows, and numerical accuracy checks.

\subsection{Router-gradient interventions}
\label{sec:controlled-scale}

We test whether removing a class-conditional router-gradient contrast
delays recovery in paired Transformer-MoE training.

\paragraph{Matched training and intervention}
We use a six-layer Transformer-MoE with eight experts per layer and top-2
routing, trained on AG News text \cite{zhang2015character} with an
autoregressive next-token objective and class-balanced sampling. AdamW is
used without auxiliary load balancing or router noise. Appendix~\ref{sec:sm-experimental-configuration} gives the model configuration.

In each middle layer, the target expert \(j\) has the smallest mean
pre-dispatch probability among experts with mean at least \(0.05\); the
comparator \(c\) has the largest remaining mean. We select the input class
on which the target's mean probability is largest. After controlled
suppression, a scalar router-bias shift calibrates its overall mean to
\(\epsilon\in\{0.005,0.01,0.02,0.04\}\)
(Appendix~\ref{app:subsection-4-2}). Selection, calibration, and recovery
measurement share a fixed 512-sequence validation sample; optimizer
updates use training data.

Write \(G_i^{\rm sel}\) for the selected-class gradient contribution for
router-parameter row \(i\), including its bias. Both branches start from the same
calibrated parameters and optimizer state and use identical minibatches. We
compare unmodified training (control) with a branch that orthogonally projects the
class-conditional row gradients onto the subspace
\(G_j^{\rm sel}=G_c^{\rm sel}\):
\[
 G_j^{\rm sel},G_c^{\rm sel}
 \leftarrow
 \tfrac12(G_j^{\rm sel}+G_c^{\rm sel}),
 \tfrac12(G_j^{\rm sel}+G_c^{\rm sel}).
\]
Projection precedes global norm clipping and AdamW. The paired response
therefore includes subsequent optimizer and representation dynamics.

\paragraph{Confirmed recovery and inference}
Each branch runs for \(\tau=1{,}024\) steps, with evaluation every 64 steps.
Confirmed recovery occurs at the second of two consecutive evaluations with
overall mean target probability at least \(0.05\). We average confirmation
times truncated at \(\tau\) over the 16 layer--initial-mass combinations within
each seed. Averaging over seeds gives the restricted mean time to confirmed
recovery (RMST) \cite{royston2013restricted};
\(\Delta\mu_\tau\) denotes its projection-minus-control difference. We chose RMST as a post hoc summary of the recovery trajectories. Inference uses
16 paired seeds without token dropping and eight with capacity factor \(1.25\);
95\% confidence intervals (CIs) use percentile-bootstrap resamples of paired
seeds with their complete collections of layer--initial-mass combinations.

Projection delays confirmed recovery under both dispatch policies.
The restricted mean increases by \(44.25\) steps without dropping
(95\% CI \([31.25,57.50]\)) and \(59.00\) steps with capacity constraints
(\([39.50,81.00]\)); both paired-seed intervals lie above zero
(\Cref{tab:mechanism-validation-results}).

\begin{table}[t]
\caption{Effect of router-gradient projection on restricted mean time to confirmed recovery through 1,024 steps.
Confirmation requires mean target probability \(\geq0.05\) at two consecutive
64-step evaluations. Differences are projection minus control, with 95\%
paired-seed bootstrap CIs. The no-drop and capacity settings use 16 and eight
paired seeds; sign-test \(p\)-values are exact and two-sided, excluding ties.}
\label{tab:mechanism-validation-results}
\centering
\small
\setlength{\tabcolsep}{3pt}
\begin{tabular}{@{}lrrlr@{}}
\toprule
Dispatch & Control & Projection & \shortstack[l]{RMST difference\\(95\% CI)} & \shortstack{Sign-test\\\(p\)}\\
\midrule
No drop & 796.25 & 840.50 & $+44.25\;[31.25,57.50]$ & $9.77\times10^{-4}$\\
Capacity \(1.25\) & 810.50 & 869.50 & $+59.00\;[39.50,81.00]$ & $7.81\times10^{-3}$\\
\bottomrule
\end{tabular}
\end{table}

\begin{figure}[t]
\centering
\includegraphics[width=\linewidth]{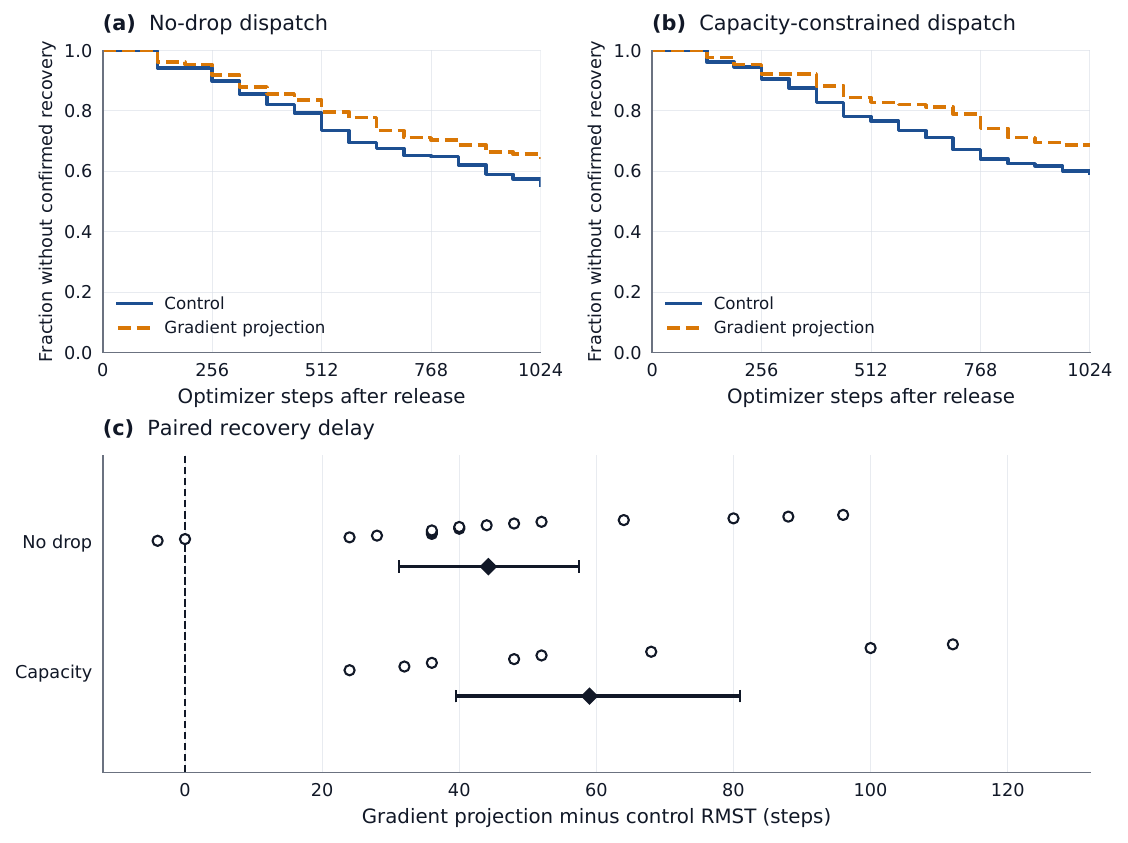}
\caption{Router-gradient projection delays confirmed recovery.
\textbf{(a), (b)} Fractions without confirmed recovery under no-drop and
capacity-constrained dispatch (factor \(1.25\)), using 16 and eight paired
seeds, respectively. Curves weight seeds and their layer--initial-mass combinations
equally; their areas through 1,024 steps define RMST.
\textbf{(c)} Circles show paired-seed RMST differences (projection minus
control); diamonds and bars give means and 95\% paired-seed bootstrap CIs.}
\label{fig:recovery-interventions-overview}
\end{figure}

By the observation horizon, projection reduces the confirmed fraction by
\(8.59\) percentage points under both dispatch policies
(\Cref{fig:recovery-interventions-overview}). Post hoc analyses
at three alternative probability thresholds preserve the direction, with
all pointwise 95\% paired-seed intervals above zero. Descriptive RMST
contrasts are also positive at each of the four initial probabilities in
both dispatch settings.

\paragraph{Fixed-token log-odds}
We measure accumulated log-odds on fixed, initially subthreshold tokens
from the selected class, using eight paired seeds per dispatch setting.
The no-drop trajectories are a subset of the RMST sample; capacity token
measurements use separate executions with the same seed values as the
RMST analysis. Initial states and token identities match between branches.
Define
\[
 A_j^N(x)=\operatorname{logit}(p_j^N(x))
 -\operatorname{logit}(p_j^0(x)),
 \qquad
 B(x)=\operatorname{logit}(0.05)
 -\operatorname{logit}(p_j^0(x)).
\]
A token meets its threshold exactly when \(A_j^N(x)\geq B(x)\).
We pool retained tokens across layers within each initial mass, then weight
masses and seeds equally; inference resamples whole paired seeds. These token-level measurements use the selected class at a single
evaluation; confirmed recovery uses the overall sample at two consecutive
evaluations.

Projection reduces both accumulated log-odds and the fraction of tokens
meeting the threshold at step 1,024 (\Cref{tab:logodds-intervention}).
Together with the RMST contrasts, these measurements link a specified
gradient intervention to slower overall recovery and smaller token-level
log-odds gains. The discrete identity in
\Cref{prop:discrete-logodds-diagnostic} expresses those gains in the
coordinates that determine individual threshold crossing.

\begin{table}[t]
\caption{Log-odds response of initially subthreshold tokens at step 1,024.
Differences are projection minus control, with 95\% paired-seed bootstrap
CIs from eight seeds per setting. Threshold attainment is evaluated for each token before averaging;
pp denotes percentage points.}
\label{tab:logodds-intervention}
\centering
\small
\begin{tabular}{@{}L{.23\linewidth}L{.35\linewidth}L{.35\linewidth}@{}}
\toprule
Dispatch & \shortstack[l]{Cumulative log-odds\\difference (95\% CI)} & \shortstack[l]{Difference in fraction\\meeting threshold (pp; 95\% CI)}\\
\midrule
No drop & $-0.194\;[-0.252,-0.133]$ & $-1.06\;[-1.64,-0.54]$\\
Capacity \(1.25\) & $-0.228\;[-0.304,-0.140]$ & $-2.28\;[-3.71,-0.82]$\\
\bottomrule
\end{tabular}
\end{table}

Appendices~\ref{app:subsection-4-2} and \ref{sec:sm-rmst}--\ref{sm:sec:logodds-trajectories} give the intervention
procedure, recovery summaries, and fixed-token trajectories.

\FloatBarrier

\section{Discussion and conclusions}
\label{sec:scope}

The normal-velocity decomposition isolates two effects: covariance among
complementary experts and the suppressed route's relative gradient advantage. Their combined coefficient distinguishes recovery regimes within
the same squared-Jacobian geometry. Positive two-sided power bounds along
recovery paths, uniform over initial masses, give logarithmic or power-law
recovery times; constructed fields reproduce these orders.

Face selection and expert adaptation determine the setting in which
recovery takes place. A stationary-loss barrier identifies the supporting
face for fixed experts. In the common-target square-loss model with freely
trainable outputs, slow adaptation approximates fixed-expert trajectories
on finite horizons, while every fixed positive rate gives an interior
limit. Explicit tails and probability floors quantify the interior
convergence for each fixed rate.

In paired Transformer-MoE training, removing a class-conditional
router-gradient contrast delays confirmed recovery and reduces accumulated
log-odds on fixed tokens under both dispatch policies. These intervention
effects connect a change in training gradients to the motion needed for
recovery. The discrete log-odds identity relates the token-level response
to threshold crossing under the actual optimizer updates.

Estimating recovery coefficients in naturally evolving systems and testing
their predictions on new trajectories would extend the continuous analysis.
Independent evaluation inputs and norm-matched gradient controls would
further assess how the neural response transfers across inputs and depends
on the chosen gradient direction.

\FloatBarrier
\bibliographystyle{abbrvnat}
\bibliography{references}

@article{shazeer2017outrageously,
  title={Outrageously large neural networks: The sparsely-gated mixture-of-experts layer},
  author={Shazeer, Noam and Mirhoseini, Azalia and Maziarz, Krzysztof and Davis, Andy and Le, Quoc and Hinton, Geoffrey and Dean, Jeff},
  journal={arXiv preprint arXiv:1701.06538},
  year={2017}
}

@article{fedus2022switch,
  title={Switch transformers: Scaling to trillion parameter models with simple and efficient sparsity},
  author={Fedus, William and Zoph, Barret and Shazeer, Noam},
  journal={Journal of Machine Learning Research},
  volume={23},
  number={120},
  pages={1--39},
  year={2022}
}

@inproceedings{lewis2021base,
  title={Base layers: Simplifying training of large, sparse models},
  author={Lewis, Mike and Bhosale, Shruti and Dettmers, Tim and Goyal, Naman and Zettlemoyer, Luke},
  booktitle={International Conference on Machine Learning},
  pages={6265--6274},
  year={2021},
  organization={PMLR}
}

@article{zhou2022mixture,
  title={Mixture-of-experts with expert choice routing},
  author={Zhou, Yanqi and Lei, Tao and Liu, Hanxiao and Du, Nan and Huang, Yanping and Zhao, Vincent and Dai, Andrew M and Le, Quoc V and Laudon, James and others},
  journal={Advances in Neural Information Processing Systems},
  volume={35},
  pages={7103--7114},
  year={2022}
}

@article{zoph2022st,
  title={St-moe: Designing stable and transferable sparse expert models},
  author={Zoph, Barret and Bello, Irwan and Kumar, Sameer and Du, Nan and Huang, Yanping and Dean, Jeff and Shazeer, Noam and Fedus, William},
  journal={arXiv preprint arXiv:2202.08906},
  year={2022}
}

@inproceedings{puigcerver2024sparse,
  title={From sparse to soft mixtures of experts},
  author={Puigcerver, Joan and Riquelme Ruiz, Carlos and Mustafa, Basil and Houlsby, Neil},
  booktitle={International Conference on Learning Representations},
  volume={2024},
  pages={28435--28445},
  year={2024}
}

@article{wang2024auxiliary,
  title={Auxiliary-loss-free load balancing strategy for mixture-of-experts},
  author={Wang, Lean and Gao, Huazuo and Zhao, Chenggang and Sun, Xu and Dai, Damai},
  journal={arXiv preprint arXiv:2408.15664},
  year={2024}
}

@article{han2025theoretical,
  title={A Theoretical Framework for Auxiliary-Loss-Free Load Balancing of Sparse Mixture-of-Experts in Large-Scale AI Models},
  author={Han, XY and Zhong, Yuan},
  journal={arXiv preprint arXiv:2512.03915},
  year={2025}
}

@article{xue2024openmoe,
  title={Openmoe: An early effort on open mixture-of-experts language models},
  author={Xue, Fuzhao and Zheng, Zian and Fu, Yao and Ni, Jinjie and Zheng, Zangwei and Zhou, Wangchunshu and You, Yang},
  journal={arXiv preprint arXiv:2402.01739},
  year={2024}
}

@article{nguyen2024temperature,
  title={Is temperature sample efficient for softmax gaussian mixture of experts?},
  author={Nguyen, Huy and Akbarian, Pedram and Ho, Nhat},
  journal={arXiv preprint arXiv:2401.13875},
  year={2024}
}

@article{nguyen2024least,
  title={On least square estimation in softmax gating mixture of experts},
  author={Nguyen, Huy and Ho, Nhat and Rinaldo, Alessandro},
  journal={arXiv preprint arXiv:2402.02952},
  year={2024}
}

@article{kawata2025mixture,
  title={Mixture of experts provably detect and learn the latent cluster structure in gradient-based learning},
  author={Kawata, Ryotaro and Matsutani, Kohsei and Kinoshita, Yuri and Nishikawa, Naoki and Suzuki, Taiji},
  journal={arXiv preprint arXiv:2506.01656},
  year={2025}
}

@inproceedings{dikkala2023benefits,
  title={On the benefits of learning to route in mixture-of-experts models},
  author={Dikkala, Nishanth and Ghosh, Nikhil and Meka, Raghu and Panigrahy, Rina and Vyas, Nikhil and Wang, Xin},
  booktitle={Proceedings of the 2023 Conference on Empirical Methods in Natural Language Processing},
  pages={9376--9396},
  year={2023}
}

@article{liao2025guided,
  title={Guided by the Experts: Provable Feature Learning Dynamic of Soft-Routed Mixture-of-Experts},
  author={Liao, Fangshuo and Kyrillidis, Anastasios},
  journal={arXiv preprint arXiv:2510.07205},
  year={2025}
}

@article{varre2026gradient,
  title={Gradient Flow Polarizes Softmax Outputs towards Low-Entropy Solutions},
  author={Varre, Aditya and Rofin, Mark and Flammarion, Nicolas},
  journal={arXiv preprint arXiv:2603.06248},
  year={2026}
}

@article{kiselev2026minimal,
  title={A Minimal Bifurcation Model of Load Imbalance in a Softmax Mixture-of-Experts Router},
  author={Kiselev, OM},
  journal={arXiv preprint arXiv:2605.29121},
  year={2026}
}

@article{wang2026myth,
  title={The myth of expert specialization in moes: Why routing reflects geometry, not necessarily domain expertise},
  author={Wang, Xi and Hayou, Soufiane and Nalisnick, Eric},
  journal={arXiv preprint arXiv:2604.09780},
  year={2026}
}

@article{ahrac2026routers,
  title={Routers Learn the Geometry of Their Experts: Geometric Coupling in Sparse Mixture-of-Experts},
  author={Ahrac, Sagi and Hochwald, Noya and Geva, Mor},
  journal={arXiv preprint arXiv:2605.12476},
  year={2026}
}

@article{guo2026advancing,
  title={Advancing expert specialization for better moe},
  author={Guo, Hongcan and Lu, Haolang and Nan, Guoshun and Chu, Bolun and Zhuang, Jialin and Yang, Yuan and Che, Wenhao and Cao, Xinye and Leng, Sicong and Cui, Qimei and others},
  journal={Advances in Neural Information Processing Systems},
  volume={38},
  pages={48767--48809},
  year={2026}
}

@article{sha2026sparse,
  title={Sparse Mixture-of-Experts Routing in Visual Diffusion Transformers: Diagnosis, Boundary Calibration and Evolutionary Roadmap from Routing Collapse to Selective Deadlock},
  author={Sha, Haiying},
  journal={arXiv preprint arXiv:2605.19378},
  year={2026}
}

@inproceedings{cheng2026emoe,
  title={EMoE: Eigenbasis-guided routing for mixture-of-experts},
  author={Cheng, Anzhe and Duan, Shukai and Li, Shixuan and Yin, Chenzhong and Cheng, Mingxi and Nazarian, Shahin and Thompson, Paul and Bogdan, Paul},
  booktitle={ICASSP 2026-2026 IEEE International Conference on Acoustics, Speech and Signal Processing (ICASSP)},
  pages={1081--1085},
  year={2026},
  organization={IEEE}
}

@article{gu2026path,
  title={Path-Constrained Mixture-of-Experts},
  author={Gu, Zijin and Likhomanenko, Tatiana and Thilak, Vimal and Ramapuram, Jason and Jaitly, Navdeep},
  journal={arXiv preprint arXiv:2603.18297},
  year={2026}
}

@article{chen2026concentration,
  title={From Concentration to Differentiation and Back: Routing Effective Rank in MoE Reasoning Cohorts},
  author={Chen, Kang and Zhao, Sihan and Cao, Yixin and Jiang, Yu-Gang},
  journal={arXiv preprint arXiv:2609.06403},
  year={2026}
}

@article{yoon2026experts,
  title={When Are Experts Misrouted? Counterfactual Routing Analysis in Mixture-of-Experts Language Models},
  author={Yoon, Youngsik and Wang, Siwei and Chen, Wei and Ok, Jungseul},
  journal={arXiv preprint arXiv:2605.07260},
  year={2026}
}

@article{jacobs1991adaptive,
  title={Adaptive mixtures of local experts},
  author={Jacobs, Robert A and Jordan, Michael I and Nowlan, Steven J and Hinton, Geoffrey E},
  journal={Neural computation},
  volume={3},
  number={1},
  pages={79--87},
  year={1991},
  publisher={MIT Press}
}

@article{jordan1994hierarchical,
  title={Hierarchical mixtures of experts and the EM algorithm},
  author={Jordan, Michael I and Jacobs, Robert A},
  journal={Neural computation},
  volume={6},
  number={2},
  pages={181--214},
  year={1994},
  publisher={MIT Press}
}

@article{lepikhin2020gshard,
  title={Gshard: Scaling giant models with conditional computation and automatic sharding},
  author={Lepikhin, Dmitry and Lee, HyoukJoong and Xu, Yuanzhong and Chen, Dehao and Firat, Orhan and Huang, Yanping and Krikun, Maxim and Shazeer, Noam and Chen, Zhifeng},
  journal={arXiv preprint arXiv:2006.16668},
  year={2020}
}

@inproceedings{du2022glam,
  title={Glam: Efficient scaling of language models with mixture-of-experts},
  author={Du, Nan and Huang, Yanping and Dai, Andrew M and Tong, Simon and Lepikhin, Dmitry and Xu, Yuanzhong and Krikun, Maxim and Zhou, Yanqi and Yu, Adams Wei and Firat, Orhan and others},
  booktitle={International conference on machine learning},
  pages={5547--5569},
  year={2022},
  organization={PMLR}
}

@article{hazimeh2021dselect,
  title={Dselect-k: Differentiable selection in the mixture of experts with applications to multi-task learning},
  author={Hazimeh, Hussein and Zhao, Zhe and Chowdhery, Aakanksha and Sathiamoorthy, Maheswaran and Chen, Yihua and Mazumder, Rahul and Hong, Lichan and Chi, Ed},
  journal={Advances in Neural Information Processing Systems},
  volume={34},
  pages={29335--29347},
  year={2021}
}

@article{chen1995optimal,
  title={Optimal rate of convergence for finite mixture models},
  author={Chen, Jiahua},
  journal={The Annals of Statistics},
  pages={221--233},
  year={1995},
  publisher={JSTOR}
}

@article{allman2009identifiability,
  title={Identifiability of parameters in latent structure models with many observed variables},
  author={Allman, Elizabeth S and Matias, Catherine and Rhodes, John A},
  year={2009}
}

@book{hofbauer1998evolutionary,
  title={Evolutionary games and population dynamics},
  author={Hofbauer, Josef and Sigmund, Karl and others},
  volume={1},
  year={1998},
  publisher={Cambridge university press Cambridge}
}

@book{sandholm2010population,
  title={Population games and evolutionary dynamics},
  author={Sandholm, William H},
  volume={25},
  year={2010},
  publisher={MIT press Cambridge}
}

@book{rockafellar1997convex,
  title={Convex analysis},
  author={Rockafellar, R Tyrrell},
  volume={28},
  year={1997},
  publisher={Princeton university press}
}

@article{beck2003mirror,
  title={Mirror descent and nonlinear projected subgradient methods for convex optimization},
  author={Beck, Amir and Teboulle, Marc},
  journal={Operations Research Letters},
  volume={31},
  number={3},
  pages={167--175},
  year={2003},
  publisher={Elsevier}
}

@article{zhang2015character,
  title={Character-level convolutional networks for text classification},
  author={Zhang, Xiang and Zhao, Junbo and LeCun, Yann},
  journal={Advances in neural information processing systems},
  volume={28},
  year={2015}
}

@inproceedings{anguita2013public,
  title={A public domain dataset for human activity recognition using smartphones.},
  author={Anguita, Davide and Ghio, Alessandro and Oneto, Luca and Parra, Xavier and Reyes-Ortiz, Jorge Luis and others},
  booktitle={Esann},
  volume={3},
  number={1},
  pages={3--4},
  year={2013}
}

@article{polyak1964gradient,
  title={Gradient methods for solving equations and inequalities},
  author={Polyak, Boris T},
  journal={USSR Computational Mathematics and Mathematical Physics},
  volume={4},
  number={6},
  pages={17--32},
  year={1964},
  publisher={Elsevier}
}

@article{wright1993identifiable,
  title={Identifiable surfaces in constrained optimization},
  author={Wright, Stephen J},
  journal={SIAM Journal on Control and Optimization},
  volume={31},
  number={4},
  pages={1063--1079},
  year={1993},
  publisher={SIAM}
}

@article{royston2013restricted,
  title={Restricted mean survival time: an alternative to the hazard ratio for the design and analysis of randomized trials with a time-to-event outcome},
  author={Royston, Patrick and Parmar, Mahesh KB},
  journal={BMC medical research methodology},
  volume={13},
  number={1},
  pages={152},
  year={2013},
  publisher={Springer}
}

@misc{uci_handwritten_digits,
  author = {Alpaydin, E. and Kaynak, C.},
  title = {Optical Recognition of Handwritten Digits},
  year = {1998},
  howpublished = {UCI Machine Learning Repository},
  url = {https://uci.edu}
}

@article{merity2016pointer,
  title={Pointer sentinel mixture models},
  author={Merity, Stephen and Xiong, Caiming and Bradbury, James and Socher, Richard},
  journal={arXiv preprint arXiv:1609.07843},
  year={2016}
}

@article{chi2022representation,
  title={On the representation collapse of sparse mixture of experts},
  author={Chi, Zewen and Dong, Li and Huang, Shaohan and Dai, Damai and Ma, Shuming and Patra, Barun and Singhal, Saksham and Bajaj, Payal and Song, Xia and Mao, Xian-Ling and others},
  journal={Advances in Neural Information Processing Systems},
  volume={35},
  pages={34600--34613},
  year={2022}
}

@article{chen2022towards,
  title={Towards understanding the mixture-of-experts layer in deep learning},
  author={Chen, Zixiang and Deng, Yihe and Wu, Yue and Gu, Quanquan and Li, Yuanzhi},
  journal={Advances in neural information processing systems},
  volume={35},
  pages={23049--23062},
  year={2022}
}

@inproceedings{mei2020global,
  title={On the global convergence rates of softmax policy gradient methods},
  author={Mei, Jincheng and Xiao, Chenjun and Szepesvari, Csaba and Schuurmans, Dale},
  booktitle={International conference on machine learning},
  pages={6820--6829},
  year={2020},
  organization={PMLR}
}

@inproceedings{zhurethinking,
  title={Rethinking Convergence in MoE Training: The Role of Routing Sparsity},
  author={Zhu, Weihao and Shi, Long and Wei, Kang and Wang, Zhe and Zhou, Yipeng and Zhang, Haixia},
  booktitle={Forty-third International Conference on Machine Learning}
}

@inproceedings{wang2025remoe,
  title={Remoe: Fully differentiable mixture-of-experts with relu routing},
  author={Wang, Ziteng and Zhu, Jun and Chen, Jianfei},
  booktitle={International Conference on Learning Representations},
  volume={2025},
  pages={59486--59507},
  year={2025}
}

\clearpage
\appendix
\numberwithin{figure}{section}
\numberwithin{table}{section}
\section*{Appendices}
\addcontentsline{toc}{section}{Appendices}

\paragraph{Contents}
These appendices contain auxiliary theory, additional numerical results, and details of the input-conditioned routing and recovery experiments.
\begin{center}
\small
\begin{tabular}{@{}L{.12\linewidth}L{.73\linewidth}r@{}}
\toprule
Appendix & Contents & Page\\
\midrule
\ref{sec:sm-theory} & Further theoretical results & \pageref{sec:sm-theory}\\
\ref{sec:sm-face-selection-results} & Additional numerical results & \pageref{sec:sm-face-selection-results}\\
\ref{sec:sm-conditioned} & Input-conditioned routing & \pageref{sec:sm-conditioned}\\
\ref{sec:sm-recovery-interventions} & Recovery interventions & \pageref{sec:sm-recovery-interventions}\\
\ref{sec:sm-statistical-checks} & Recovery analysis & \pageref{sec:sm-statistical-checks}\\
\bottomrule
\end{tabular}
\end{center}

\section{Further theoretical results}
\label{sec:sm-theory}
We use the notation and norm conventions of the main text. Here
\(p\in\Delta^{n-1}\) is the routing vector for \(n\) experts,
\(E\in\R^{n\times d}\) has expert outputs as rows,
\(y\in\R^d\) is the target, \(q\in\R^n\) is the logit vector, and
\(\eta\) is the expert-to-router learning-rate ratio. In particular,
\[
 \mathcal L(p,E)=\|E^\top p-y\|^2,
 \qquad J(p)=\operatorname{diag}(p)-pp^\top,
 \qquad M(p)=J(p)^2.
\]
As in the main text, these flows optimize the half-loss
\(\mathcal H:=\mathcal L/2\).
The fixed-expert router flow and joint expert equation are
\begin{align}
 \dot p&=-M(p)E(E^\top p-y), \label{sm:eq:fixed-flow}\\
 \dot E_i&=-\eta p_i(E^\top p-y). \notag
\end{align}
We write \(g=E(E^\top p-y)\) and \(s=J(p)g\) for the probability
and logit gradients of \(\mathcal H\).

\subsection{Convergence to an edge in the three-expert model}
\label{app:subsection-1-1}
Consider the three-expert triangle
\[
 E_1=(0,0),\qquad E_2=(1,0),\qquad E_3=(0,1),
 \qquad y=(x,-h),\quad 0<x<1,\ h>0.
\]

\begin{theorem}[Convergence to an edge with three experts]
\label{sm:thm:canonical-edge}
Let \(p(t)\) solve \eqref{sm:eq:fixed-flow} from the simplex interior. Define
\[
 p_{23}:=\left(0,\frac{1+x+h}{2},\frac{1-x-h}{2}\right)
\]
when \(x+h<1\), and omit this point otherwise. Let \(e_i\) denote the
\(i\)th simplex vertex. If \(\mathcal L(v,E)>\mathcal L(p(0),E)\)
for every
\(v\in\{e_1,e_2,e_3\}\cup
 (\{p_{23}\}\text{ if }x+h<1)\), then
\[
 p(t)\to(1-x,x,0).
\]
In particular, \(p_3(t)\to0\); hence for every \(\delta>0\),
\(p_3(t)<\delta\) for all sufficiently large \(t\).
\end{theorem}

\begin{proof}
Here \(E^\top p=(p_2,p_3)\) and
\(g=(0,p_2-x,p_3+h)\). An interior stationary point would require all three
components of \(g\) to agree, while \(p_3+h>0\). Every vertex is
stationary. In the relative interior of edge \(\{2,3\}\), the condition
\(g_2=g_3\) gives \(p_{23}\), which exists there exactly when \(x+h<1\).
On edge \(\{1,3\}\), equality \(g_1=g_3\) would require
\(p_3=-h\), so its relative interior contains no stationary point.
On edge \(\{1,2\}\), equality \(g_1=g_2\) gives \((1-x,x,0)\).
Loss dissipation and
LaSalle's invariance principle place the \(\omega\)-limit set in this stationary
set, and the strict loss inequalities select
\((1-x,x,0)\) as its sole member.
\end{proof}

\subsection{Separation of stationary outputs}
\label{app:subsection-1-2}
For fixed \(E,y\), let
\(z^\ast=\operatorname{Proj}_{\operatorname{conv}\{E_i\}}y\) and
\(r_\ast=\|z^\ast-y\|\). Let \(\ell_{\rm comp}\) be the lowest suboptimal stationary loss defined
in the main text.

\begin{proposition}[Separation of stationary outputs]
Define
\[
 \mathcal Z_{\rm stat}
 :=\{E^\top p:p\in\Delta^{n-1},\;
        J(p)E(E^\top p-y)=0\}.
\]
This set is finite. If it contains an output different from \(z^\ast\), set
\[
 \sigma_\ast:=\min_{z\in\mathcal Z_{\rm stat},\,z\neq z^\ast}
 \|z-z^\ast\|.
\]
Then
\[
 r_\ast^2<r_\ast^2+\sigma_\ast^2\leq\ell_{\rm comp}.
\]
Thus loss below \(r_\ast^2+\sigma_\ast^2\) satisfies the barrier condition in the main face-selection theorem. Computing \(\sigma_\ast\) requires solving a global polynomial problem.
\end{proposition}

\begin{proof}
On each stationary support, \(E^\top p\) is the unique projection of \(y\)
onto the convex hull of the expert outputs in that support, so only finitely many stationary outputs
occur. For any suboptimal one, the projection inequality at \(z^\ast\) gives
\[
 \|z-y\|^2-r_\ast^2
 \geq\|z-z^\ast\|^2\geq\sigma_\ast^2.
\]
Minimizing over suboptimal stationary outputs proves the result.
\end{proof}

\paragraph{Computing the geometric bounds}
Support enumeration considers up to \(2^n-1\) nonempty restricted hulls to
identify the stationary outputs and their loss barrier. The hitting-time
constant \(\mu(c,\delta)\) in \Cref{prop:outside-face-hitting} is a global
minimum of the squared logit-gradient norm over the compact loss sublevel
with outside-face mass at least \(\delta\). Once an exposing direction is
available, \Cref{cor:exposed-face-mass} instead converts a measured loss excess directly
into an outside-face mass bound. \Cref{sec:sm-face-selection-detail} evaluates stationary-loss
separation across fixed-expert geometries and initializations.

\subsection{Explicit finite-horizon stability constants}
\label{app:subsection-1-3}
\begin{lemma}[Uniform bounds for slow expert motion]
\label{sm:lem:explicit-stability}
Let \(T>0\) be a fixed horizon and \(E^0\) the common initial expert
matrix for the fixed and joint flows, with common routing state \(p(0)\).
Set \(\rho=\|E^{0\top}p(0)-y\|\) and \(Y=\|y\|\).
For the joint square-loss flow with \(0\leq\eta\leq1\), set
\[
 K=\|E^0\|_{\rm F}+T\rho,\qquad
 L=6K(K+Y)+K^2,\qquad K_E=2K+Y.
\]
The fixed-expert solution \(p_0\) and the joint solution satisfy, for
\(0\leq t\leq T\),
\[
 \|E_\eta(t)-E^0\|_{\rm F}\leq\eta\rho t,
 \qquad
 \|p_\eta(t)-p_0(t)\|\leq\eta K_E\rho\,\Psi_L(t),
\]
where
\[
 \Psi_L(t)=\begin{cases}
 (e^{Lt}-1-Lt)/L^2,&L>0,\\
 t^2/2,&L=0.
 \end{cases}
\]
In particular, \(C_T=\rho T+K_E\rho\Psi_L(T)\) is an admissible constant
in the main text's finite-horizon bound. Its dependence on \(T\) is
explicit through \(\Psi_L(T)\) and the expert-norm bound \(K\).
\end{lemma}
\begin{proof}
For the residual \(R_\eta(t)=E_\eta(t)^\top p_\eta(t)-y\),
joint loss dissipation implies \(\|R_\eta(t)\|\leq\rho\). Since
\(\|p_\eta\|\leq1\), integrating
\(\|\dot E_\eta\|_{\rm F}=\eta\|p_\eta\|\|R_\eta\|\)
gives the first bound and \(\|E_\eta\|_{\rm F}\leq K\).
For simplex points, \(\|J(p)\|\leq1\) and
\(\|J(p)-J(\widetilde p)\|\leq3\|p-\widetilde p\|\), hence
\(\|J(p)^2-J(\widetilde p)^2\|\leq6\|p-\widetilde p\|\).
For \(f(p,E)=-J(p)^2E(E^\top p-y)\), these inequalities give
\[
 \|f(p,E)-f(\widetilde p,\widetilde E)\|
 \leq L\|p-\widetilde p\|+K_E\|E-\widetilde E\|_{\rm F}
\]
when both expert norms are at most \(K\). Here the two terms in \(L\)
bound the changes in \(J(p)^2\) and in the residual, respectively;
\(K_E\) bounds the two occurrences of \(E\). Bounded expert and logit
velocities give continuation on \([0,T]\). With
\(d(t)=\|p_\eta(t)-p_0(t)\|\), the probability error obeys
\[
 d(t)\leq\int_0^t\bigl(Ld(s)+\eta K_E\rho s\bigr)\,ds
 \leq\eta K_E\rho\int_0^t e^{L(t-s)}s\,ds.
\]
Evaluating the integral gives \(\Psi_L(t)\), including its continuous
extension at \(L=0\).
\end{proof}

\subsection{Interior convergence for PL residual losses}
\label{app:subsection-1-4}
The interior-convergence argument extends to residual losses with lower and upper gradient bounds. The lower bound is the Polyak--\L{}ojasiewicz (PL) inequality \cite{polyak1964gradient} for a nonnegative loss with minimum value zero.
\begin{theorem}[Interior convergence for a PL residual-loss class]
\label{sm:thm:pl-interior}
Let \(\phi:\R^d\to[0,\infty)\) have locally Lipschitz gradient and suppose
that, for \(\mu,L>0\),
\[
 2\mu\phi(r)\leq\|\nabla\phi(r)\|^2\leq2L\phi(r),
 \qquad r\in\R^d.
\]
For \(R=E^\top p-y\), consider
\[
 \dot q=-J(p)E\nabla\phi(R),\qquad
 \dot E=-\eta p\nabla\phi(R)^\top,
 \qquad p=\operatorname{Softmax}(q),
\]
with \(\eta>0\) and finite \(q(0),E(0)\). Taking
\(\phi(r)=\|r\|^2/2\) recovers the main half-loss flow, with
\(\mu=L=1\). Define
\[
 \Phi_0:=\phi(R(0)),\quad
 G_0:=\frac{n\sqrt{2L\Phi_0}}{\eta\mu},\quad
 B_\phi:=\|E(0)\|_{\rm op}+\eta G_0,\quad
 C_\phi:=B_\phi G_0.
\]
Then the solution is global,
\[
 \phi(R(t))\leq\Phi_0e^{-2\eta\mu t/n},
\]
and \(E(t)\to E_\infty\), \(q(t)\to q_\infty\), with
\[
 \|E(t)-E_\infty\|_{\rm F}\leq\eta G_0e^{-\eta\mu t/n},
 \qquad
 \|q(t)-q_\infty\|\leq C_\phi e^{-\eta\mu t/n}.
\]
Moreover,
\[
 p_i(t)\geq
 \frac{p_i(0)}{p_i(0)+e^{2C_\phi}(1-p_i(0))}>0
\]
for every expert and every time, so the router converges to an interior point.
\end{theorem}

\begin{proof}
Writing \(D=\nabla\phi(R)\), gradient-flow dissipation and
\(\|p\|^2\geq1/n\) give
\[
 \frac{d}{dt}\phi(R)
 =-\|J(p)ED\|^2-\eta\|p\|^2\|D\|^2
 \leq-\frac{2\eta\mu}{n}\phi(R).
\]
The gradient upper bound therefore yields
\[
 \int_t^\infty\|D(\tau)\|\,d\tau
 \leq G_0e^{-\eta\mu t/n}.
\]
Integrating the expert and logit equations gives the stated tails and precludes
finite-time divergence of the logits. Finally \(\|q(t)-q(0)\|\leq C_\phi\); bounding every logit
difference by \(2C_\phi\) gives the lower bound on the probabilities and interior limit.
\end{proof}

For \(\phi(r)=\|r\|^2/2\), these constants reduce to
\(G_0=n\|R(0)\|/\eta\),
\(B_\phi=\|E(0)\|_{\rm op}+n\|R(0)\|\), and
\(C_\phi=nB_\phi\|R(0)\|/\eta\), as in the main square-loss theorem.

\subsection{A log-ratio criterion for interior convergence}
\label{app:subsection-1-5}
\begin{proposition}[Interior convergence from accumulated logit-gradient differences]
\label{sm:prop:finite-score-freeze}
Let \((p(t),E(t))\) be any global solution of
\(\dot p=-J(p)^2E(t)(E(t)^\top p-y)\)
that remains in \(\operatorname{int}(\Delta^{n-1})\) for every finite time.
Define \(g(t)=E(t)(E(t)^\top p(t)-y)\), \(s(t)=J(p(t))g(t)\), and
\(A_{ij}(t):=\int_0^t(s_i(\tau)-s_j(\tau))\,d\tau\).
If \(A_{ij}(t)\) has a finite limit, then \(p_i(t)/p_j(t)\) converges to a
finite positive limit. If \(\sup_{t\geq0}|A_{in}(t)|<\infty\) for every \(i<n\),
then
\[
 \inf_{t\geq0}\min_i p_i(t)>0.
\]
If every \(A_{in}(t)\) converges, then
\(p(t)\to p_\infty\in\operatorname{int}(\Delta^{n-1})\).
In particular, bounded experts and
\(\int_0^\infty\|E(t)^\top p(t)-y\|\,dt<\infty\) imply this interior
convergence.
\end{proposition}

\begin{proof}
The softmax log-ratio identity gives
\(\frac{p_i(t)}{p_j(t)}=\frac{p_i(0)}{p_j(0)}e^{-A_{ij}(t)}\).
Uniformly bounded reference differences keep all ratios \(p_i/p_n\) between
positive constants, so normalization yields the uniform positive lower bound. If
they converge, set
\(r_i=\lim_{t\to\infty}p_i(t)/p_n(t)\in(0,\infty)\). Then
\[
 p_n(t)\to\frac{1}{1+\sum_{i<n}r_i},\qquad
 p_i(t)\to\frac{r_i}{1+\sum_{k<n}r_k}.
\]
Finally,
\(|s_i-s_j|\leq2\|E\|_{\rm op}\|E^\top p-y\|\), so the last hypotheses
make every logit-gradient difference absolutely integrable.
\end{proof}

\subsection{Cumulative log-odds thresholds for recovery}
\label{app:subsection-1-6}
\begin{corollary}[Cumulative log-odds threshold for coupled recovery]
\label{sm:cor:coupled-recovery-budget}
For the coupled-flow log-odds identity in the main text, write
\(a_j(t)=d\operatorname{logit}(p_j(t))/dt\). Fix a recovery threshold
\(\theta\in(0,1)\), let \(p_j(0)=\epsilon\in(0,\theta)\), and set
\(T_\theta=\inf\{t\geq0:p_j(t)\geq\theta\}\). If
\(T_\theta<\infty\) and \(p_j(t)\in(0,1)\) on
\([0,T_\theta]\), then
\begin{equation*}
 \int_0^{T_\theta}a_j(t)\,dt
 =B_{\theta,\epsilon}
 :=\log\frac{\theta(1-\epsilon)}{(1-\theta)\epsilon}
 =\log\frac1\epsilon+O_\theta(1).
\end{equation*}
If \(a_j(t)\leq\bar a\) on \([0,T_\theta]\) for some \(\bar a>0\), then
\(T_\theta\geq B_{\theta,\epsilon}/\bar a\). Moreover, for any horizon
\(T\),
\[
 \int_0^T a_j(t)_+\,dt<B_{\theta,\epsilon},
 \qquad a_j(t)_+:=\max\{a_j(t),0\},
\]
implies \(p_j(t)<\theta\) for every \(t\in[0,T]\).
\end{corollary}

\begin{proof}
At the first hitting time, continuity and the main text's exact log-odds
identity give
\(\int_0^{T_\theta}a_j(t)\,dt=B_{\theta,\epsilon}\).
The upper bound on \(a_j\) yields
\(B_{\theta,\epsilon}\leq\bar aT_\theta\). If recovery occurred by time
\(T\), the same identity at its first hitting time would imply
\(B_{\theta,\epsilon}\leq\int_0^T a_j(t)_+\,dt\); contraposition proves the
last claim.
\end{proof}

\subsection{First-order group separation with independent logits}
\label{app:subsection-1-7}
\begin{proposition}[First-order separation with independent logits]
\label{sm:prop:independent-escape}
Let \(k=1,\ldots,K\) index groups with targets \(y_k\in\R^d\).
Give each group its own logits \(q^{(k)}\) and routing probabilities
\(p^{(k)}=\operatorname{Softmax}(q^{(k)})\), and fix the experts.
Under the half-loss gradient flow for each group,
\(\dot p^{(k)}=-M(p^{(k)})E(E^\top p^{(k)}-y_k)\).
At a shared point \(p^{(1)}=\cdots=p^{(K)}=p_0\), define
\(\chi(p_0,E):=\max_{k,\ell}\|M(p_0)E(y_k-y_\ell)\|\).
The routing velocities differ across groups exactly when
\(\chi(p_0,E)>0\), in which case
\[
 \max_{k,\ell}\|p^{(k)}(t)-p^{(\ell)}(t)\|
 =\chi(p_0,E)t+o(t).
\]
If the groupwise routing distributions converge and their limiting outputs
have different minimal hull faces, those distributions are distinct.
Objective weights rescale time, while the
groupwise vector fields determine persistent between-group routing variation.
\end{proposition}

\begin{proof}
At a shared point,
\(\dot p^{(k)}(0)-\dot p^{(\ell)}(0)=M(p_0)E(y_k-y_\ell)\),
which proves the separation criterion and first-order expansion. Outputs with
different minimal hull faces are distinct, so they cannot be represented by
the same routing distribution. Multiplying a group objective by
a positive weight \(w_k\) rescales time and leaves its phase portrait unchanged.
\end{proof}

\Needspace{6\baselineskip}
\section{Additional numerical results}
\label{sec:sm-face-selection-results}
These studies compare selection, adaptation, and recovery with their
continuous predictions. Solver refinement and independent quadrature
assess numerical accuracy.

\subsection{Face selection}
\label{sec:sm-face-selection-detail}
\Cref{tab:face-selection-evidence} reports the representative fixed-expert trajectories.
The initialization sweep below examines sublevel entry and face selection
across the triangle simplex.

\begin{figure}[t]
\centering
\includegraphics[width=.54\linewidth]{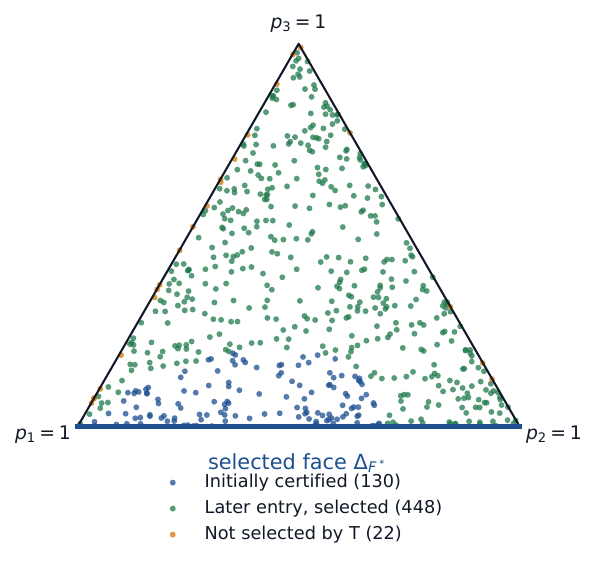}
\caption{Face selection across triangle initializations.
Colors distinguish initial loss below the stationary-loss barrier,
later sublevel entry with face selection, and trajectories that have
not selected the face by \(T=2000\).}
\label{sm:fig:full-face-sweep}
\end{figure}

For a projection in the relative interior of the full hull, take
\(E_1=(0,0)\), \(E_2=(1,0)\), \(E_3=(0,1)\), and
\(y=(0.30,0.25)\). The unique minimizing route is
\(p^\ast=(0.45,0.30,0.25)\). Initialization at \((0.55,0.24,0.21)\)
has loss \(0.0052\), below the lowest suboptimal stationary loss \(0.0625\).
DOP853 integrations to \(t=2000\), with relative tolerances \(10^{-10}\)
and \(10^{-12}\), give terminal probability errors at most
\(2.55\times10^{-12}\). This illustrates the full-hull case of the
main face-selection theorem.

\paragraph{Geometric configurations}
The four configurations use the following expert orderings, targets, and
initial probabilities:
\begin{enumerate}
\item Triangle: expert rows \((0,0),(1,0),(0,1)\),
\(y=(0.35,-0.20)\), and \(p(0)=(1/3,1/3,1/3)\).
\item Pentagon: expert rows \((-1,0),(1,0),(1.3,1),(0,1.5),(-1.2,1)\),
\(y=(0.25,-0.50)\), and \(p(0)=(0.04,0.05,0.25,0.38,0.28)\).
\item Redundant hexagon face: expert rows
\((-1,0),(0,0),(1,0),(1.2,0.9),(0,1.5),(-1.2,0.9)\),
\(y=(0.20,-0.40)\), and \(p(0)=(0.03,0.03,0.04,0.24,0.38,0.28)\).
\item Cube: the first four rows are
\((-1,-1,0),(1,-1,0),(1,1,0),(-1,1,0)\); the next four have the same
first two coordinates and third coordinate one. Here \(y=(0,0,-0.40)\)
and \(p(0)=(0.025,0.025,0.025,0.025,0.20,0.22,0.25,0.23)\).
The boundary control has support \(\{1,2,7\}\) and weights
\((0.50,0.18,0.32)\) on those experts.
\end{enumerate}
For each of four fixed-expert geometries, we enumerate the \(2^n-1\) nonempty supports, identify the stationary supports, and compute \(\ell_{\rm comp}\) from their restricted-hull projections. The classical fourth-order Runge--Kutta method (RK4) with
time step \(\Delta t=0.02\) runs to \(t=2000\), including a control initialized at a
suboptimal stationary point. \Cref{tab:face-selection-evidence} gives entry times and terminal errors.

We draw 600 triangle initializations from a Dirichlet distribution with
concentration parameter \(0.7\) in each coordinate. Under step halving,
all face-selection classifications agree: 130 start below the loss
barrier, 448 enter later and select the face, and 22 have not selected the
face by \(T=2000\). Of the last group, seven have entered the sublevel set
and 15 have not. The maximum terminal routing discrepancy is below
\(10^{-11}\).

\subsection{Expert-adaptation time scales}
\label{sec:sm-time-scales}
Joint-flow simulations compare residual decay and interior stabilization
with the fixed-horizon approximation by fixed-expert routing.

\paragraph{Joint training and PL residual losses}
The joint square-loss study uses three configurations \((n,d,\eta)\):
\((3,2,0.1)\), \((5,4,0.5)\), and \((8,6,1)\).
Each configuration uses five Gaussian initializations and three RK4 step
sizes over \(0\leq t\leq60\). All trajectories satisfy the residual and probability bounds. The largest coarse-to-fine terminal routing error is below
\(3\times10^{-5}\).

The nonquadratic experiment uses
\(\phi(r)=\frac12\|r\|^2+0.35\sum_k(1-\cos r_k)\).
This loss is \(0.65\)-strongly convex and \(1.35\)-smooth, so the PL bounds
hold with \((\mu,L)=(0.65,1.35)\). Using the same configurations and
integration scheme as above, all finest-step trajectories satisfy the
loss and probability bounds; the coarse-to-fine terminal routing error
is again below \(3\times10^{-5}\).

The time-scale comparison in \Cref{fig:reduced-flow-overview}(b,c) uses the triangle
configuration above, RK4 step \(0.02\), and
\(\eta\in\{0.01,0.02,0.04,0.08,0.16\}\). The maximum routing error on
\([0,40]\) has fitted log--log slope \(1.03\), consistent with the
fixed-horizon approximation.

\begin{figure}[t]
\centering
\includegraphics[width=\linewidth]{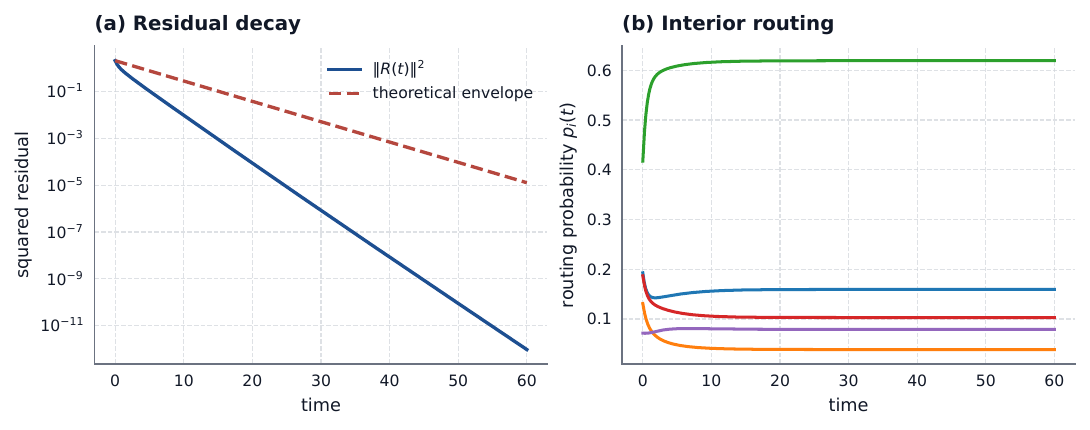}
\caption{Residual decay and interior routing under joint training.
\textbf{(a)} Squared residual norm and its exponential envelope.
\textbf{(b)} Expert routing probabilities along the same trajectory.}
\label{sm:fig:exact-gradient-convergence}
\end{figure}

\subsection{Recovery-time scaling}
\label{sec:sm-recovery-scaling}
We vary the effective recovery coefficient within squared-Jacobian flow. Comparing numerical threshold events with
the resulting hitting-time integrals tests the conditional recovery-time
classification.

Two constructed vector fields realize the structural recovery regimes. For
the field with nonzero covariance \(\Lambda_j\), let \(w\)
denote the conditional weights on the face \(p_j=0\), with all sums
below over \(k\neq j\), and set
\[
 \nu(w):=\sum_k w_k^3-\left(\sum_k w_k^2\right)^2,\qquad
 g_j=0,\qquad
 g_k=10^{-3}\frac{w_k-\|w\|_2^2}{\nu(w)}.
\]
On the domain \(\nu(w)>0\), this gives \(\Delta_j=0\) and
\(\Lambda_j=10^{-3}\) identically. Here \(\nu(w)\) is the variance of
\(w_k\) under \(k\sim w\); the vector field can become singular as this
variance vanishes. The initial
conditional weights are proportional to the entries of
\((1^8,2^8,\ldots,(n-1)^8)\);
\(\epsilon\in\{0.005,0.0075,0.01,0.02,0.03\}\), \(\theta=0.05\), and
classical RK4 uses \(\Delta t=0.25\). For the zero-covariance field, the complementary
weights are uniform, \(g_j=-1\), \(g_k=0\) for \(k\neq j\), so
\(\Lambda_j=0\), \(\Delta_j=1\), and the complementary weights remain uniform.
Here
\(\epsilon\in\{0.00025,0.0005,0.001,0.002,0.004\}\) and
\(\Delta t=0.1\), with the same threshold. The predicted hitting times are
\(H_1(\epsilon,\theta)/10^{-3}\) and
\(H_2(\epsilon,\theta)/(1+1/(n-1))\), respectively.

All 20 trajectories recover, with maximum relative hitting-time error
below \(4\times10^{-7}\). The covariance-driven times fit
\(\log(1/\epsilon)\), while the zero-covariance log--log slope is
\(-1.02\).

\paragraph{Constructed fields with different orders of vanishing}
A field with a prescribed recovery coefficient illustrates the theorem for
\(\alpha\in\{0,0.25,0.5,0.75,1\}\). Let \(x=p_j\), initialize
the \(n-1\) complementary weights uniformly, and write
\(\kappa_n=1+1/(n-1)\). Choose
\[
 g_j=-\frac{x^{\alpha-1}}{\kappa_n},\qquad g_k=0\quad(k\neq j).
\]
The complementary weights remain uniform, so \(\Lambda_j=0\),
\(\Delta_j=x^{\alpha-1}/\kappa_n\), and \(\Gamma_j=x^\alpha\) exactly.
Thus the predicted time is \(H_{1+\alpha}(\epsilon,\theta)\).
The field is smooth on every recovery interval
\(0<\epsilon\leq x\leq\theta<1\) and satisfies the coefficient
comparison there with constants equal to one. For \(\alpha<1\),
\(g_j\) diverges as \(x\downarrow0\). The preceding field with nonuniform complementary weights realizes logarithmic recovery through a fixed positive covariance.

For positive \(\alpha\leq1\), we fit the negative slope of
\(\log T_\theta\) against \(\log\epsilon\) on
\([10^{-12},10^{-6}]\) and \([10^{-12},10^{-9}]\). Narrowing the window
gives fitted exponents \(0.251\), \(0.500\), \(0.750\), and \(1.000\)
for the four positive orders, with agreement across expert counts.
At \(\alpha=0\), fitting \(T_\theta\) against \(\log(1/\epsilon)\)
gives slope \(1.00\). \Cref{sm:fig:reduced-extensions} shows the fitting-window
effect.

For \(\alpha>1\), the same construction has bounded probability
gradients, with \(g_j\to0\) as \(x\downarrow0\). We test
\(\alpha=1.5,2\) at \(n=4,8\) over initial masses from
\(10^{-9}\) to \(10^{-5}\), using the transformed-time scheme below.
The fitted exponents agree with the imposed orders to within \(10^{-5}\),
and the maximum relative quadrature error is below \(2\times10^{-12}\).

\paragraph{Optimization geometry and changing advantages}
The two-expert comparison evaluates the
separated-variable formulas in the two-expert recovery corollary for
\(\epsilon\in[10^{-4},10^{-1}]\) and
\(\theta=0.25\). The \(J^2\) slope is \(-0.996\). For \(\epsilon=10^{-4}\),
let \(T_{J^2}\) and \(T_{\rm proj}\) denote the corresponding hitting times;
then \(\epsilon T_{J^2}=0.5006\) and
\(T_{\rm proj}=0.4998\).

Starting from \(\epsilon=0.08\) with threshold \(\theta=0.30\), classical
RK4 with \(\Delta t=0.002\) directly integrates the \(J^2\), \(J\), and
Euclidean projected probability ODEs under an oscillation with positive
drift, a zero-mean oscillation, and a subthreshold oscillation. The instantaneous advantage changes sign over each observation horizon. Recovery occurs in four of
the nine combinations. All nine outcomes agree with whether the accumulated
advantage \(\sup_t A(t)\) reaches the flow-specific threshold over the observation horizon. The
numerical error in \(I_r(\epsilon,p(t))=A(t)\) stays below
\(10^{-13}\). For the positive-drift advantage, the sign changes before
recovery under all three flows.

\subsection{Numerical accuracy and refinement}
\label{sec:sm-experiment-sources}
\paragraph{Normal-velocity identity}
Direct evaluation of \(-J(p)^2g\) checks the normal-velocity identity
at 256 random simplex states per expert count, with \(n=4,8\),
Dirichlet-distributed probabilities and Gaussian gradients. The identity
holds to floating-point precision.

\paragraph{Quadrature and event-time agreement}
The integration uses \(n=4,8\) and \(\theta=0.05\),
with initial masses \(\epsilon\in\{10^{-12},10^{-11},\ldots,10^{-3}\}\).
We evolve all logits using transformed time \(s\), where \(ds=x^\alpha dt\).
We also rescale physical time as \(\tau=\epsilon^\alpha t\), giving
\[
 \frac{dq}{ds}=\frac{e_j-p}{\kappa_n},\qquad
 \frac{d\tau}{ds}=\left(\frac{\epsilon}{x}\right)^\alpha.
\]
The explicit eighth-order Runge--Kutta method DOP853 uses relative tolerance \(10^{-11}\), absolute tolerance \(10^{-13}\),
and maximum steps \(0.5\) and \(0.25\) in \(s\). Event detection at
\(p_j=\theta\) is compared with independent quadrature of
\[
 \epsilon^\alpha H_{1+\alpha}(\epsilon,\theta)
 =\int_0^{\log(\theta/\epsilon)}
 \frac{e^{-\alpha z}}{(1-\epsilon e^z)^2}\,dz.
\]
Across 200 trajectories, both quadrature differences and changes under
step halving remain below \(10^{-12}\) in relative terms.

\begin{figure}[t]
\centering
\includegraphics[width=\linewidth]{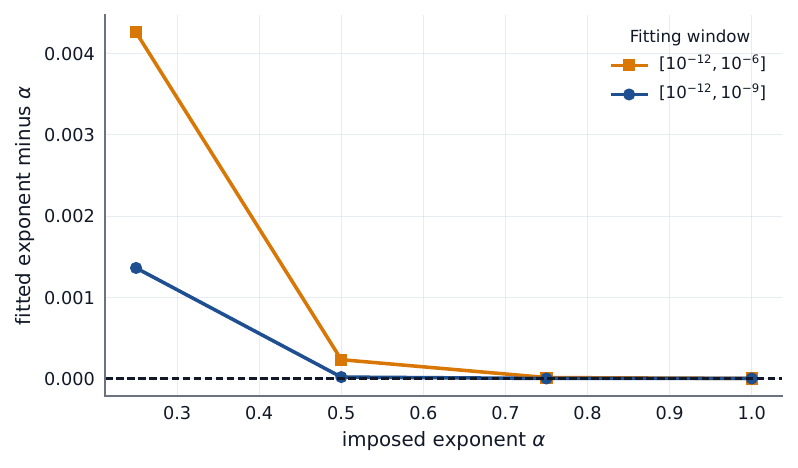}
\caption{Fitting-window effects on recovery exponents.
Fitted exponent errors decrease as the upper end of the initial-mass
window approaches zero. Results use \(n=4\); \(n=8\) agrees at plotting
precision.}
\label{sm:fig:reduced-extensions}
\end{figure}

\begin{figure}[t]
\centering
\includegraphics[width=\linewidth]{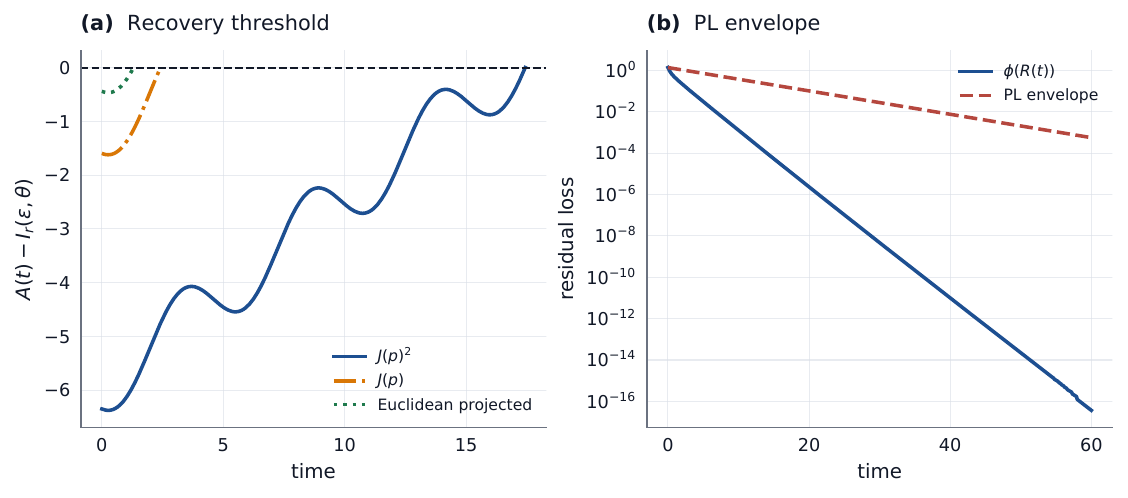}
\caption{Recovery with changing advantages and nonquadratic joint training.
\textbf{(a)} Accumulated sign-changing advantage relative to the
flow-specific recovery threshold; zero marks threshold attainment.
\textbf{(b)} Nonquadratic PL loss and its exponential envelope.}
\label{sm:fig:assumption-closure}
\end{figure}

\FloatBarrier
\section{Input-conditioned routing}
\label{sec:sm-conditioned}
For \(K\) groups and \(n\) experts, let \(P\in\R^{K\times n}\)
have routing distribution \(P_{k,:}\) in row \(k\). Define the mean between-group \(\ell^1\) deviation and the number of distinct
row-wise maximizers by
\[
\begin{aligned}
 S(P)&:=\frac1K\sum_{k=1}^K\|P_{k,:}-\bar P\|_1,
 &\bar P&:=\frac1K\sum_{k=1}^K P_{k,:},\\
 i_k^\ast&:=\min\arg\max_i P_{k,i},
 &D(P)&:=\left|\{i_k^\ast:1\leq k\leq K\}\right|.
\end{aligned}
\]
Here \(\bar P\) is the mean routing distribution, \(i_k^\ast\) selects
the smallest expert index in a tie, and set bars denote cardinality.
Thus \(S(P)=0\) exactly when all groups have the same routing distribution;
\(D(P)\) counts their distinct dominant experts. These statistics distinguish input-dependent allocation from aggregate load.
\subsection{Affine routing and shared-bias invariance}
\label{sec:shared-conditioned}
\Cref{prop:coupledbreakfreezefirstorder} identifies when feature--residual
coupling separates initially equal routing distributions. We test this
criterion and show that shared bias shifts preserve cross-group log-odds
contrasts.

\begin{remark}[A shared bias preserves cross-group log-odds contrasts]
\label{rem:shared-bias-invariant}
For any groups \(k,\ell\) and experts \(i,j\), a common expert-bias vector
\(v\in\R^n\), applied as \(\widetilde q^{(k)}=q^{(k)}+v\) with
\(\widetilde p^{(k)}=\operatorname{Softmax}(\widetilde q^{(k)})\), leaves
\[
 \log\frac{\widetilde p_i^{(k)}}{\widetilde p_j^{(k)}}
 -\log\frac{\widetilde p_i^{(\ell)}}{\widetilde p_j^{(\ell)}}
 =\log\frac{p_i^{(k)}}{p_j^{(k)}}
 -\log\frac{p_i^{(\ell)}}{p_j^{(\ell)}}
\]
unchanged: both within-group log ratios acquire the same \(v_i-v_j\).
Such a bias can change aggregate load and top-\(k\) assignments without
changing this cross-group contrast. Individual probabilities, log-odds
thresholds, and subsequent dynamics depend on the shifted probabilities.
\end{remark}

\paragraph{Initial routing separation}
The affine-router calculation compares predicted and observed initial
separation rates in 64 random feature--residual configurations and
16 nullspace controls. The median relative prediction error is
\(2\times10^{-5}\); every control has zero separation to numerical
precision (\Cref{sm:fig:feature-slope}).

\begin{figure}[t]
\centering
\includegraphics[width=.80\linewidth]{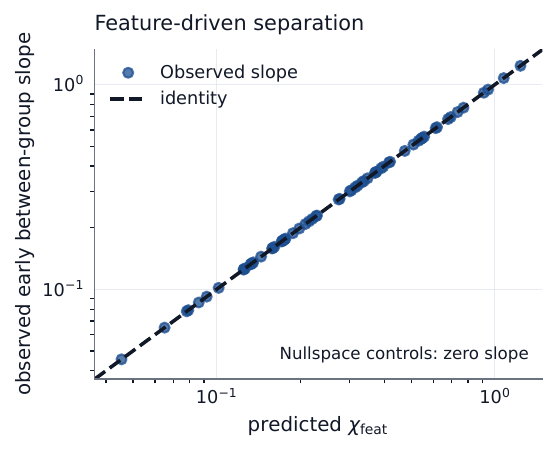}
\caption{Initial routing separation under a shared affine router.
Observed slopes agree with \(\chi_{\rm feat}\) in random feature--residual
configurations; nullspace controls have zero first-order separation.}
\label{sm:fig:feature-slope}
\end{figure}

\subsection{Models and training}
\label{sec:sm-experimental-configuration}
The synthetic routing and synthetic mini-batch multilayer perceptron (MLP)
studies use Adam; the Digits and UCI HAR MLPs and all Transformers use AdamW.
The larger Transformer-MoE has six layers, width 384, six attention heads,
and eight two-layer GELU experts of hidden width 768 per layer, with top-2
routing and approximately \(50\)--\(55\) million parameters. Training updates
the full model in bfloat16 precision. The four-domain and independent
AG News comparisons use approximately \(6.4\times10^7\) and
\(5.9\times10^6\) target tokens per condition, respectively.

The smaller studies use six experts: a one-block Transformer on AG News
and MLPs on Digits and UCI HAR. Complete architectures, training settings,
and random seeds are provided in the accompanying code.

\subsection{Routing structure across datasets}
\label{sec:sm-routing-comparisons}
The synthetic studies compare shared-global and input-conditioned routers,
including balancing, entropy, and capacity variants. In the three
shared-global variants, the group routing distributions coincide, giving
\(S(P)=0\) and \(D(P)=1\). Input conditioning produces distinct
allocations across groups (\Cref{sm:tab:routing-structure}).

\begin{table}[t]
\caption{Routing structure across synthetic and external datasets.
Values are means \(\pm\) standard deviations over ten paired seeds,
except twelve for UCI HAR. Single values denote invariant statistics.
The synthetic shared-global row summarizes the identical routing-structure
statistics of softmax, load-balanced, and top-2 variants.}
\label{sm:tab:routing-structure}
\centering
\small
\setlength{\tabcolsep}{5pt}
\begin{tabular}{@{}L{.48\linewidth}rr@{}}
\toprule
Router & \shortstack{Mean \(\ell^1\)\\deviation} & \shortstack{Dominant\\experts}\\
\midrule
\multicolumn{3}{@{}l}{\textit{Synthetic data}}\\
Shared global & \(0\) & \(1\)\\
Input-conditioned, softmax & \(1.012\pm0.120\) & \(3.1\pm0.5\)\\
Input-conditioned, entropy & \(0.305\pm0.047\) & \(3.2\pm0.4\)\\
Input-conditioned, capacity & \(1.072\pm0.084\) & \(4\)\\
\midrule
\multicolumn{3}{@{}l}{\textit{AG News}}\\
Global top-2 & \(0\) & \(1\)\\
Conditioned top-2 & \(0.266\pm0.058\) & \(2.8\pm0.9\)\\
Conditioned, load balancing & \(0.241\pm0.060\) & \(3.7\pm0.5\)\\
\midrule
\multicolumn{3}{@{}l}{\textit{Digits}}\\
Global top-2 & \(0\) & \(1\)\\
Conditioned top-2 & \(0.804\pm0.114\) & \(4.2\pm0.7\)\\
\midrule
\multicolumn{3}{@{}l}{\textit{UCI HAR}}\\
Global top-2 & \(0\) & \(1\)\\
Conditioned top-2 & \(1.150\pm0.067\) & \(4.3\pm0.5\)\\
\bottomrule
\end{tabular}
\end{table}

The synthetic mini-batch MLP comparison gives paired deviation increases
of \(1.037\) (95\% CI \([0.959,1.109]\)) for softmax and \(1.150\)
(\([1.052,1.246]\)) for top-2 routing, positive in all ten pairs.

The external-data studies use a stratified AG News subset
\cite{zhang2015character}, a stratified Digits split with 1,347 training
and 450 test samples \cite{uci_handwritten_digits}, and the official
subject-disjoint UCI HAR split with 7,352 training and 2,947 test windows
\cite{anguita2013public}. HAR features are standardized using training-set
statistics. Routing deviation was the prespecified outcome. Paired
inference uses seeds as the independent units. The deviation increases
are positive in all ten AG News and Digits pairs and all twelve HAR pairs.

\subsection{Model scale and capacity constraints}
\label{sec:sm-controlled-scale}
Each study uses matched initialization and data order; estimates are
reported separately across datasets and model scales.

\paragraph{Datasets and train--test splits}
The four-domain corpus contains ArXiv abstracts, Project Gutenberg books,
U.S. Government Publishing Office text, and WikiText-103
\cite{merity2016pointer}. Whole documents
are assigned to disjoint train, validation, and shift splits; GPT-2 tokenized
fragments have lengths 256 or 512. Domain labels support balanced sampling and
held-out routing summaries.

\paragraph{Routing differences across scales}
The conditioned-minus-shared routing contrast remains positive in the
small and large four-domain models, with and without capacity constraints,
and in the independent AG News implementation
(\Cref{sm:tab:controlled-scale-effects}). All evaluated middle layers have
positive mean contrasts, and all seed-level contrasts are positive.

\begin{table}[t]
\caption{Mean between-group \(\ell^1\) deviation differences across settings
(input-conditioned minus shared global), averaged over evaluated middle
layers. Intervals are paired-seed bootstrap 95\% CIs. Each study uses eight
seeds except the four-domain capacity comparison, whose four-seed mean is
reported descriptively. Capacity factor is \(1.25\).}
\label{sm:tab:controlled-scale-effects}
\centering
\small
\begin{tabular}{@{}L{.49\linewidth}rr@{}}
\toprule
Study & Mean difference & 95\% CI\\
\midrule
Four-domain small model & 0.444 & [0.404,0.478]\\
Four-domain large model & 0.236 & [0.217,0.255]\\
Four-domain large, capacity & 0.174 & ---\\
AG News, no drop & 0.102 & [0.081,0.128]\\
AG News, capacity & 0.102 & [0.087,0.117]\\
\bottomrule
\end{tabular}
\end{table}

\section{Recovery interventions}
\label{sec:sm-recovery-interventions}
We compare clamp removal with continued suppression, then estimate the
effect of removing a class-conditional router-gradient contrast. Paired
branches share their calibrated state and training sequence, so their
recovery-time difference measures the response to this gradient intervention.
\subsection{Release from routing suppression}
\label{app:subsection-4-1}
Auxiliary AG News experiments compare clamp removal with continued
suppression from the same checkpoint, optimizer state, and minibatch
sequence. Each dispatch setting uses eight paired seeds, four middle
layers, and 1,024 post-release steps. Recovery is confirmed by two
consecutive 64-step evaluations with mean target probability at least
\(0.05\). Clamp removal yields confirmation in 24 of 32 no-drop and
29 of 32 capacity-constrained seed--layer combinations, compared with
none under continued suppression.

\begin{figure}[t]
\centering
\includegraphics[width=\linewidth]{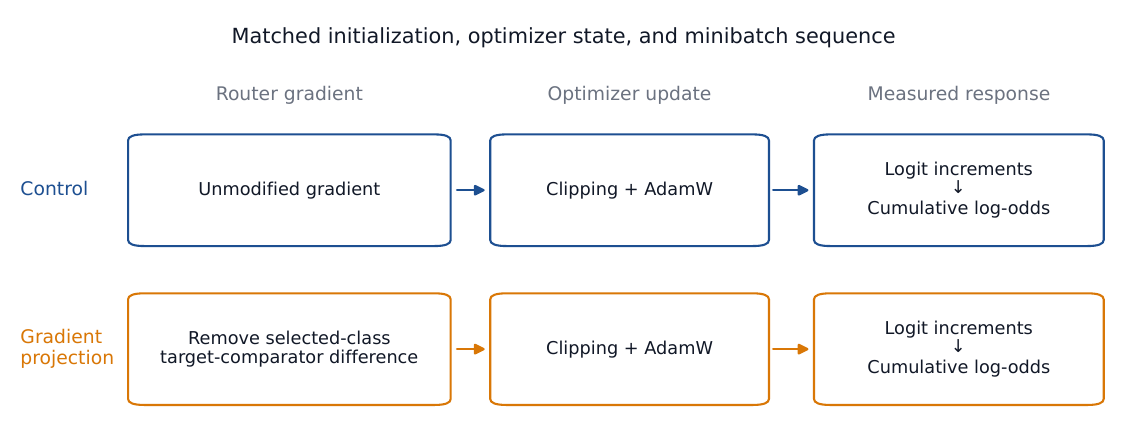}
\caption{Paired training with and without class-conditional gradient projection.
Branches share the calibrated state and minibatches; projection acts before
clipping and AdamW.}
\label{sm:fig:gradient-workflow}
\end{figure}

\subsection{Class-conditional router-gradient intervention}
\label{app:subsection-4-2}
The intervention uses 16 paired seeds without drops and eight at
capacity factor \(1.25\), evaluating the four middle layers. It uses the six-layer model in
\Cref{sec:sm-experimental-configuration}, with noise-free affine logits and
renormalized top-2 dispatch. Training minimizes masked autoregressive
cross-entropy without auxiliary balancing on 119,800 AG News sequences.
AdamW uses batch size four, learning rate \(2\times10^{-4}\), weight
decay \(0.01\), and gradient-norm clipping at one. Targets, comparators,
and selected classes follow the main text's rules after 29,950 training
steps; a 1,024-step suppression interval then produces the common state
for calibration.

\paragraph{Initial-probability calibration}
The clamp is removed for calibration. For each middle layer in network order,
let \(p_{j,a}\) be the target probability at active token \(a\) in the fixed
512-sequence validation sample, containing \(N_{\rm tok}\) active
(nonpadding) tokens. The desired mean is
\(\epsilon\in\{0.005,0.01,0.02,0.04\}\). We shift only that layer's target-router bias
by \(d\), solving
\[
 \frac1{N_{\rm tok}}\sum_{a=1}^{N_{\rm tok}}
 \frac{e^d p_{j,a}}{1-p_{j,a}+e^d p_{j,a}}=\epsilon.
\]
Bisection calibrates the layers in network order, remeasuring probabilities
after each upstream bias shift. The final absolute mean-probability error
is at most \(0.001\) in every middle layer. Calibration changes only
target-router biases. All branches receive the same calibrated
parameters and suppression-stage optimizer state. Selection, calibration,
and recovery measurement use the same validation sample; subsequent
updates use training data with matched class-balanced minibatches.

\paragraph{Gradient intervention}
Projection replaces the selected-class target and comparator gradient
rows by their average, leaving other rows and other classes' gradient contributions unchanged. It acts before global clipping and
AdamW, so the measured response includes ensuing optimizer and
representation dynamics. Capacity-constrained dispatch
drops about 35\% of assignments across the control and projection branches.

\Cref{sec:sm-rmst} defines confirmed recovery and the statistical analysis.

\section{Recovery analysis}
\label{sec:sm-statistical-checks}
Confirmed recovery uses the overall mean routing probability. The
token-level analysis follows fixed inputs from the selected class.
\subsection{Restricted mean time to confirmed recovery}
\label{sec:sm-rmst}
We index seeds by \(s\), the four middle layers by
\(\ell\in\{1,2,3,4\}\), and the initial probabilities
\((0.005,0.01,0.02,0.04)\) by \(m\in\{1,2,3,4\}\), in that order.
For condition \(b\in\{\mathrm{ctrl},\mathrm{proj}\}\) (control or gradient
projection), let the indicator \(I^b_{s,\ell,m,k}\in\{0,1\}\) equal one when the mean target probability is at least \(0.05\) at step \(64k\).
The confirmation time is
\[
 T^{\rm conf,b}_{s,\ell,m}
 =64\inf\{k\geq2:I^b_{s,\ell,m,k-1}=I^b_{s,\ell,m,k}=1\},
\]
with \(\inf\varnothing=\infty\). We observe \(k=1,\ldots,16\), with
right-censoring at \(\tau=1{,}024\) when confirmation is pending. Define
\[
 Y_s^b=\frac1{16}\sum_{\ell,m}\min\{T^{\rm conf,b}_{s,\ell,m},\tau\},
 \qquad \mu_\tau^b=\mathbb E_s[Y_s^b],
 \qquad \Delta\mu_\tau=\mathbb E_s[Y_s^{\rm proj}-Y_s^{\rm ctrl}].
\]
The expectation is over seeds, with equal weight on the 16 fixed
layer--initial-mass combinations. Equivalently, \(\mu_\tau^b\) averages
their seed-marginal survival integrals
\(\int_0^\tau\Pr_s(T^{\rm conf,b}_{s,\ell,m}>t)\,dt\).
Complete follow-up to \(\tau\) makes the empirical integral the mean of
truncated confirmation times. Confirmation exactly at \(\tau\) contributes
\(\tau\) to this mean and counts as an observed event.

Confidence intervals use 100,000 percentile-bootstrap resamples of paired
seeds, retaining all 16 layer--initial-mass combinations within each seed. Exact two-sided
sign tests exclude ties. No-drop and capacity estimates use 16 and eight
paired seeds, respectively.

\subsection{Sensitivity to the probability threshold}
\label{sm:sec:threshold-sensitivity}
We apply the same confirmation rule to the saved trajectories at
\(\theta\in\{0.04,0.05,0.06,0.08\}\). The additional thresholds are a
post hoc sensitivity analysis with pointwise paired-seed bootstrap
intervals. At \(\theta=0.04\), the largest initial mass equals the
threshold, so confirmation measures sustained attainment after release.

\begin{table}[t]
\caption{Confirmed-recovery sensitivity to the probability threshold.
RMST differences are projection minus control in optimizer steps, with
pointwise 95\% paired-seed bootstrap CIs. Percentages give confirmation by
step 1,024. The no-drop and capacity settings use 16 and eight seeds, with equal
weight on seeds and layer--initial-mass combinations within seeds.}
\label{sm:tab:threshold-sensitivity}
\centering
\small
\begin{tabular}{@{}llrrr@{}}
\toprule
Dispatch & \(\theta\) & \shortstack{RMST difference\\(95\% CI)} & \shortstack{Control\\recovery (\%)} & \shortstack{Projection\\recovery (\%)}\\
\midrule
No drop & 0.04 & $56.50\;[46.00,67.25]$ & 71.48 & 59.38\\
 & 0.05 & $44.25\;[31.25,57.50]$ & 44.53 & 35.94\\
 & 0.06 & $25.00\;[15.25,37.50]$ & 22.27 & 18.36\\
 & 0.08 & $10.50\;[4.50,17.50]$ & 6.64 & 5.08\\
\midrule
Capacity \(1.25\) & 0.04 & $75.50\;[55.50,96.50]$ & 65.63 & 53.13\\
 & 0.05 & $59.00\;[39.50,81.00]$ & 40.63 & 32.03\\
 & 0.06 & $30.50\;[9.50,52.00]$ & 23.44 & 17.97\\
 & 0.08 & $15.50\;[4.00,29.00]$ & 10.16 & 6.25\\
\bottomrule
\end{tabular}
\end{table}

All contrasts and pointwise intervals remain above zero
(\Cref{sm:tab:threshold-sensitivity}). At higher thresholds, fewer
trajectories confirm recovery and both restricted means approach the
observation horizon, reducing the contrast within this window.

\begin{figure}[t]
\centering
\includegraphics[width=\linewidth]{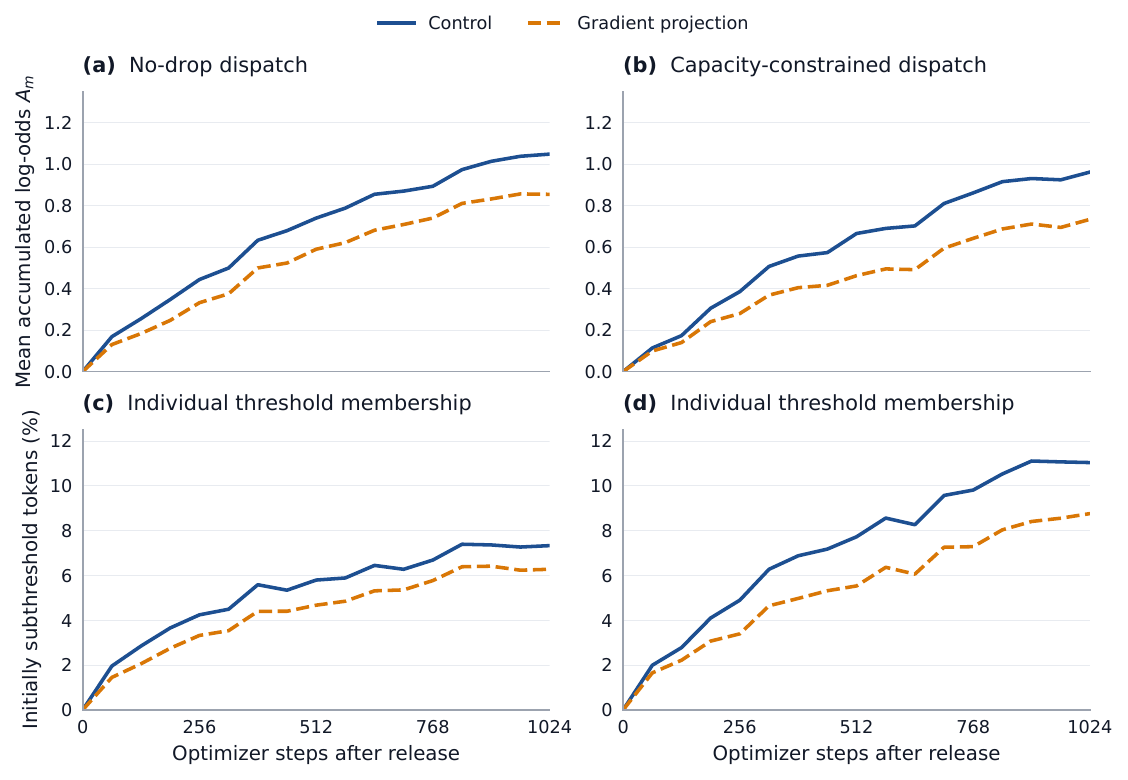}
\caption{Accumulated log-odds and fractions of tokens meeting the threshold.
\textbf{(a), (b)} Mean accumulated log-odds \(A_m\).
\textbf{(c), (d)} Fractions satisfying \(A_m(x)\geq B(x)\) at each
evaluation, computed before averaging. Each setting uses eight paired
seeds, with equal mass and seed weights; capacity factor is \(1.25\).}
\label{sm:fig:logodds-trajectories}
\end{figure}

\subsection{Recovery across initial probabilities}
\label{sec:sm-initial-mass}
\Cref{sm:tab:mechanism-mass-results} reports RMST at each initial
probability using the main analysis definition. All eight mean contrasts
favor earlier recovery under unmodified training.

\begin{table}[t]
\caption{Confirmed recovery by initial probability \(\epsilon\).
Descriptive RMST differences are projection minus control through 1,024
steps at threshold \(0.05\), averaging four layers within each seed and
then seeds equally. Percentages give confirmation by step 1,024.
No-drop and capacity settings use 16 and eight seeds, respectively;
capacity factor is \(1.25\).}
\label{sm:tab:mechanism-mass-results}
\centering
\small
\begin{tabular}{@{}llrrr@{}}
\toprule
\(\epsilon\) & Dispatch & \shortstack{RMST difference\\(steps)} & \shortstack{Control\\recovery (\%)} & \shortstack{Projection\\recovery (\%)}\\
\midrule
0.005 & No drop & $+14$ & 10.9 & 6.3\\
 & Capacity & $+22$ & 12.5 & 3.1\\
0.010 & No drop & $+33$ & 23.4 & 17.2\\
 & Capacity & $+48$ & 21.9 & 12.5\\
0.020 & No drop & $+62$ & 57.8 & 39.1\\
 & Capacity & $+94$ & 46.9 & 37.5\\
0.040 & No drop & $+68$ & 85.9 & 81.3\\
 & Capacity & $+72$ & 81.3 & 75.0\\
\bottomrule
\end{tabular}
\end{table}

\subsection{Token-level log-odds trajectories}
\label{sm:sec:logodds-trajectories}
The trajectories complement \Cref{tab:logodds-intervention} and use eight
paired seeds per dispatch setting. The no-drop measurements use a subset
of the RMST sample; capacity measurements use separate runs with the same
seed values.
We follow the first 128 active selected-class validation tokens in each
middle layer, with fixed identities across times and branches. For a token
\(x\) at evaluation index \(m\), define
\[
 A_m(x)=\operatorname{logit}(p_j^m(x))-\operatorname{logit}(p_j^0(x)),
 \qquad B(x)=\operatorname{logit}(0.05)-\operatorname{logit}(p_j^0(x)).
\]
By the discrete log-odds identity,
\(A_m(x)=\sum_{r<m}(u_j^r(x)-c_j^r(x))\), and
\(p_j^m(x)\geq0.05\) exactly when \(A_m(x)\geq B(x)\).
Here \(m\) indexes 64-step evaluations; \(u_j^m\) and \(c_j^m\) are the
corresponding logit and complementary log-sum-exp increments.

We retain tokens with \(p_j^0(x)<0.05\), determined before either branch
evolves. Within each seed and initial mass, we pool retained tokens
across layers, then weight masses and seeds equally.
\Cref{sm:fig:logodds-trajectories} shows mean accumulated changes and
the fraction of tokens meeting the threshold.

The projection branch has smaller accumulated log-odds gains and a lower
fraction of tokens meeting the threshold. These token-level responses
accompany the confirmed-recovery delay in the main analysis, connecting
the gradient intervention to both recovery measures.

\FloatBarrier

\end{document}